\documentclass{article}

\usepackage{amssymb}
\usepackage{amsthm,amsmath}
\usepackage{bbm}
\usepackage{natbib}
\usepackage[margin=1in]{geometry}
\usepackage{graphicx}
\usepackage{color}
\usepackage{url}
\usepackage{booktabs}

\usepackage[linesnumbered,lined,boxed,commentsnumbered]{algorithm2e}

\newtheorem{assumption}{Assumption}
\newtheorem{theorem}{Theorem}
\newtheorem{lemma}{Lemma}

\usepackage{MnSymbol,wasysym}

\DeclareMathOperator{\E}{\mathbb{E}}
\DeclareMathOperator{\constrained}{C}
\DeclareMathOperator{\MMD}{MMD}
\DeclareMathOperator{\Pp}{P}
\DeclareMathOperator{\L2A}{L2A}

\DeclareMathOperator{\optim}{O}
\DeclareMathOperator{\simp}{S}

\begin{document}

\title{Learning how to approve updates to machine learning algorithms in non-stationary settings}

\author{
Jean Feng\\
Department of Epidemiology and Biostatistics, University of California, San Francisco
}


\maketitle


\begin{abstract}
{
Machine learning algorithms in healthcare have the potential to continually learn from real-world data generated during healthcare delivery and adapt to dataset shifts.
As such, the FDA is looking to design policies that can autonomously approve modifications to machine learning algorithms while maintaining or improving the safety and effectiveness of the deployed models.
However, selecting a fixed approval strategy, a priori, can be difficult because its performance depends on the stationarity of the data and the quality of the proposed modifications.
To this end, we investigate a learning-to-approve approach (L2A) that uses accumulating monitoring data to learn \textit{how} to approve modifications.
L2A defines a family of strategies that vary in their ``optimism''---where more optimistic policies have faster approval rates---and searches over this family using an exponentially weighted average forecaster.
To control the cumulative risk of the deployed model, we give L2A the option to abstain from making a prediction and incur some fixed abstention cost instead.
We derive bounds on the average risk of the model deployed by L2A, assuming the distributional shifts are smooth.
In simulation studies and empirical analyses, L2A tailors the level of optimism for each problem-setting: It learns to abstain when performance drops are common and approve beneficial modifications quickly when the distribution is stable.
}
\end{abstract}


\section{Introduction}

Due to the rapid development of artificial intelligence (AI) and machine learning (ML), an increasing number of medical devices and clinical decision support software now rely on AI/ML algorithms.
The current regulatory policies of the Center of Diagnostics and Radiologic Health (CDRH) at the US Food and Drug Administration (FDA) require algorithms within such systems to be locked, i.e. the algorithms cannot change post-approval.
However, the performance of locked prediction algorithms can degrade over time due to changes in clinical practice patterns, shifts in the patient population, and more \citep{Minne2012-rb, Davis2017-bl}.
Thus, there is growing interest in deploying continuously evolving ML systems that learn from data generated during healthcare delivery \citep{Kelly2019-sa, Nestor2019-mh, Ghassemi2020-do, Li2020-lw}, and has long been the object of study in a number of fields, including continual/lifelong learning \citep{Thrun1995-qw}, online learning \citep{Shalev-Shwartz2012-vg}, meta-learning \citep{Bengio1991-sl}, and adaptive analyses \citep{Dwork2015-da}.
In contrast to locked systems, continually evolving systems have the potential to adapt to distributional shifts, better reflect real-world settings, and incorporate new advancements in ML \citep{Lee2020-os}.
Because evolving algorithms present many new regulatory challenges, the FDA recently proposed a new framework for approving modifications to AI/ML-based Software as a Medical Device (SaMD) in a discussion paper \citep{Fda2019}.

The discussion paper \citep{Fda2019} proposes companies stipulate SaMD Pre-specifications (SPS) and an Algorithm Change Protocol (ACP).
The SPS describes the types of changes the manufacturer plans to make; The ACP describes how they will ensure that the device remains safe and effective after the modifications.
Once the FDA approves the SPS and ACP, the manufacturer may deploy changes according to these documents without further intervention from the regulatory agency.
Here we study the component in the ACP that evaluates and approves modifications solely based on their performance, which we refer to as the pACP.

Designing a good pACP is not a trivial task, and \citet{Fda2019} does not provide specific guidelines or examples for how to do so.
Previous works have proposed hypothesis testing procedures for approving a candidate modification at a single time point \citep{Vergouwe2017-rh, Davis2019-hx, Davis2019-tw}.
However, one must be careful when translating these methods to multiple time points.
The error from each hypothesis test can accumulate over time and fail to protect against gradual performance deterioration, an issue more commonly known as ``bio-creep'' \citep{Fleming2008-ok}.
While one can protect against bio-creep using online hypothesis testing procedures \citep{Feng2019-lq}, such methods only control the online error-rate if the data is independently and identically distributed (IID) over the entire time period.
In practice, this IID assumption is unlikely to hold, particularly over long periods of time.
Thus, there is a need for online methods that can identify and approve beneficial modifications in non-stationary settings.

In this paper, we study how to design pACPs that control the online error-rate in the presence of distributional shifts.
We define the online error-rate as the cumulative risk.
This quantity has been extensively studied in the online learning literature and can also be viewed as a generalization of the utility function in the adaptive clinical trial literature \citep{Graf2015-yz, Simon2017-to, Simon2018-hx} to non-stationary settings.
To give the model developers flexibility in constructing algorithmic modifications, the developer is allowed to be adaptive, i.e. they may change how modifications are generated based on historical data.
Because ML algorithms may induce distributional shifts in unanticipated and undesirable ways, we allow the distributional shifts to be adaptive as well.
For example, ML algorithms can affect clinical decision-making as human experts develop trust or distrust in these systems \citep{Poursabzi-Sangdeh2018-dr, Kaur2020-ke}.
In addition, ML algorithms can induce significant feedback loops in the environments they operate in, and initially fair algorithms can become increasingly unfair \citep{Liu2018-fo, Chouldechova2018-bj, Hashimoto2018-jl, Obermeyer2019-kc}.


To protect against arbitrary distributional shifts, we frame the problem within the prediction-with-expert-advice paradigm \citep{Littlestone1994-rl, Herbster1998-vj, cesa2006prediction}.
In this setup, the pACP is thought to be playing a game against ``Nature,'' who may introduce adaptive or even adversarial distributional shifts.
At each time point, the pACP selects between candidate algorithmic modifications (experts), with the ultimate goal of minimizing the cumulative risk.
Although existing methods like the FixedShare and MarkovHedge have used this framework to continually update algorithms \citep{Kolter2005-xz, Kolter2007-au,Shalizi2011-gh}, they do not adequately address regulatory needs.
First, these methods are overly conservative in real-world settings and the derived error bounds for realistic time horizons are too wide to provide meaningful error control.
The reason is that these methods are targeted to protect against distributional shifts that are fully adversarial, which are unlikely to occur in practice.
Second, these methods do not control the absolute error---they only control the online error relative to the best sequence of modifications, which could be arbitrarily poor.

Because real-world clinical settings are neither perfectly IID over time nor replete with adversarial shifts, we consider a problem setting in between these two extremes.
In particular, we suppose the distributional shifts are bounded in terms of the maximum mean discrepancy \citep{Gretton2012-yw}, though they may still be adaptive within the prescribed bounds.
We investigate a learning-to-approve (L2A) approach that learns \textit{how} to approve algorithmic updates using labeled data sampled IID from the target population at each time point.
L2A defines a broad family of approval strategies, which encompasses strategies designed for IID data as well as adversarial setups.
It searches over this family by dynamically weighting candidate strategies according to the exponentially weighted average forecaster (EWAF), a classical solution from the prediction-with-expert-advice literature \citep{cesa2006prediction}.

Unlike previous works, L2A deploys selective prediction models, which are models that can choose to abstain from making a prediction.
This abstention option improves the reliability of ML algorithms in settings where abstaining is better than giving a misleading prediction \citep{Chow1970-uc, Cordella1995-aj, Geifman2017-rl, Bartlett2008-sn, Feng2019-of}.
Assuming there is some fixed cost $\delta > 0$ associated with abstaining, we show that the average risk of L2A never exceeds $\delta+ \epsilon$ for some non-inferiority margin $\epsilon > 0$.
This implies that L2A can safely search over unsafe approval strategies as long as the abstention option is available.
We analyze the operating characteristics of L2A in both simulation studies and real-world analyses.
In the presence of severe distributional shifts, L2A learns to approve modifications cautiously and abstain more; When the data are (relatively) stationary, it learns to approve modifications quickly and abstain less.

The paper is organized as follows.
Section~\ref{ref:setup} formalizes the problem.
Section~\ref{sec:methods} presents the family of approval strategies, how L2A searches over this family in an online manner, and risk bounds of this procedure.
We study the performance of L2A in simulation studies in Section~\ref{sec:simulations} and real-world datasets in Section~\ref{sec:real_data}.

\section{Approval process for algorithmic modifications}
\label{ref:setup}

In this section, we describe the framework and abstractions necessary to understand the approval process for modifications to AI/ML-based SaMDs.
Following \citet{Fda2019}, the modification process can be described as a cycle with three stages.
First, the manufacturer proposes a modification, which can be trained on the available monitoring data as well as external data.
This modification is added to a pool of candidate modifications and is, henceforth, locked.
Second, the pACP updates the approval status of each candidate modification based on the latest monitoring data and updates the deployed algorithm accordingly.
Third, a new batch of monitoring data is collected.
For simplicity, suppose these three stages are executed in the above order over a fixed grid of $T$ time points.

%

See Table~\ref{table:notation} in the Supplement for a summary of the notation used in this paper.

\paragraph{Patient population}
Let $\mathcal{P}$ be the family of joint distributions over the space of covariates $\mathcal{X}$ and range of possible outcomes $\mathcal{Y}$.
Given a distribution $P \in \mathcal{P}$, patients are represented by the random vector $(X, Y)$.
Let filtration $\mathcal{F}_t$ be the sigma algebra over all historical data---which includes the monitoring data, proposed models, and approvals by the pACP---observed up to time $t$.

Let $P_0 \in \mathcal{P}$ be the initial distribution.
To allow for adaptive distributional shifts, we define a sequence of distribution shift functions $\{p_t: t = 1,\dots,T\}$, where $p_t$ is a $\mathcal{F}_{t-1}$-measurable function that maps onto $\mathcal{P}$.
The distribution at time $t$ is the realized output from $p_t$, which we represent by $P_t$.
We use the notation $P_{t_1:t_2}$ to represent the uniform mixture of (realized) distributions from time $t_1$ to $t_2$, inclusive.


At each time $t$, we observe a batch of monitoring data that contains $n$ labeled observations drawn IID from $P_t$.
Let the empirical distribution of this batch be denoted $P_{t,n}$.
Note that this assumption implies that gold-standard labels can be obtained for all observations, which is not always an easy task.
For example, we exclude AI/ML-based SaMDs that predict the treatment effect, i.e. the difference in outcomes between receiving and not receiving treatment, since it is impossible to observe the counterfactual.
We leave the complexities of dealing with potential outcomes to future work.

\paragraph{Model developer}
Let $\mathcal{G}$ be a family of prediction models that map from $\mathcal{X}$ to $\tilde{\mathcal{Y}}$, where $\tilde{\mathcal{Y}}$ is the convex hull of $\mathcal{Y}$.
For example, in a binary classification problem, $\tilde{\mathcal{Y}}$ is a probability between $[0,1]$.
We represent modifications by an entirely new model in $\mathcal{G}$.
We evaluate models using a loss function $\ell: \tilde{\mathcal{Y}} \times \mathcal{Y} \mapsto [0,1]$ that is convex with respect to the first input.
The risk of a model $\tilde{g}\in \mathcal{G}$ with respect to a distribution $Q \in \mathcal{P}$ is $\E_{Q}\ell(\tilde{g}(X), Y)$.
Its empirical risk with respect to empirical distribution $Q_{n}$ is denoted $\E_{Q_n}\ell(\tilde{g}(X), Y)$.

At each time point, the model developer can propose a new modification in an adaptive fashion.
As such, we define the model development process using a sequence of functions $\{g_t: t = 1,\dots,T\}$, where $g_t$ is a $\mathcal{F}_{t - 1}$-measurable function with range $\mathcal{G}$.
The output from $g_t$ is the modification proposed at time $t$.
Let the realized output be denoted $G_t$.

\paragraph{Distributional shifts}
We assume that distributional shifts are bounded at each time point.
Similar smoothness assumptions have been advocated in previous works \citep{Rakhlin2011-du} as a more realistic representation of real-world problem settings.
We formalize this assumption in terms of the maximum mean discrepancy (MMD), a distance-measure between probability measures \citep{Gretton2012-yw}.
For any two distributions $Q_1, Q_2 \in \mathcal{P}$, define the MMD with respect to the function class $\{(x, y) \mapsto \ell(\tilde{g}(x), y) : \tilde{g} \in \mathcal{G} \}$ as
\begin{align}
\MMD(Q_1, Q_2)
:=
\sup_{\tilde{g} \in \mathcal{G}}
\left |
\E_{Q_1}  \ell(\tilde{g}(X), Y) - \E_{Q_2} \ell(\tilde{g}(X), Y)
\right |.
\label{eq:mmd}
\end{align}
So if the MMD between $P_t$ and $P_{t - 1}$ is bounded by a constant $V > 0$ for all $t = 1,...,T$, the maximum increase in risk for any candidate modification is at most $V$.
More generally, suppose there is a known window size $W$ and constant $V$ such that the MMD between $P_t$ and distributions from the last $W$ time points are bounded follows:
\begin{assumption}[Bounded maximum mean discrepancy]
	There exists some constant $V > 0$ and window size $W \in Z^+$ such that the distribution shift functions $\{p_1,\dots,p_T\}$ satisfy
	\begin{align}
	\Pr\left(
	\MMD_{\mathcal{L}}(P_{t}, P_{\max(0, t - w):t - 1}) \le V
	\quad
	\forall w = 1,\dots,W \text{ and } t = 1,\dots,T
	\right)
	= 1.
	\end{align}
	\label{assume:drift}
\end{assumption}
\noindent Let $\mathcal{D}_{V,W}$ be all sequences of distributional shift functions that satisfy Assumption~\ref{assume:drift} for a given $V$ and $W$.

\paragraph{An abstention option}
We allow the pACP to deploy prediction models with the option to abstain, also known as selective prediction models \citep{El-Yaniv2010-kd}.
We assume there is some known cost $\delta > 0$ for abstaining.
So, a selective prediction model incurs loss $\ell$ when it makes a prediction and $\delta$ when it abstains.
For notational convenience, define $G_0$ as a selective prediction model that always abstains.
In addition, define $\ell_{\delta}(z, y)$ to be the augmented loss function that is equal to $\ell(z, y)$ when $z$ is a prediction and $\delta$ when $z$ is the abstention option.


\paragraph{pACP}
At every time point $t$, the pACP evaluates all modifications proposed up to time $t$ using monitoring data up to time $t - 1$ and updates their approval status.
We generalize previous definitions of approval that only allowed ``hard'' approval of a single modification at each time point \citep{Feng2019-lq}.
Here we allow for ``soft'' approvals as well, where each candidate modification is assigned a probability weight.
From a practical viewpoint, soft approvals are not new to the regulatory space; They are analogous to outcome-adaptive randomization in clinical trials \citep{Berry1995-iz}, where the outcome in our setting is the prediction loss.
More specifically, we represent the approval status by a $(t+1)$-vector from the probability simplex $\Delta^{t + 1}$, where its first weight is the abstention rate and the remaining weights are associated with the candidate modifications.

We define the pACP as a sequence of functionals $\{\theta_t: t = 1,\dots,T\}$.
To represent the information available to the pACP at each time point, define filtration $\tilde{\mathcal{F}}_t$ as the sigma algebra over all historical data observed up to time $t - 1$ and the proposed model at time $t$.
Then, $\theta_t$ is a $\tilde{\mathcal{F}}_{t - 1}$-measurable function that maps to $\Delta^{t + 1}$.
Denote its realized output as $\hat{\theta}_t$.

Given approval status $\hat{\theta}_t$ at time $t$, the pACP deploys the stochastic selective prediction model $h_{\hat{\theta}_t}$, which abstains with probability $\hat{\theta}_{t,0}$ and predicts using an additive ensemble over the candidate modifications with probability $1 - \hat{\theta}_{t,0}$.
In particular, the ensemble is the weighted average $x \mapsto \left( \sum_{t' = 1}^{t}\hat{\theta}_{t,t'} G_{t'}(x) \right) / \left( \sum_{t' = 1}^{t} \hat{\theta}_{t,t'} \right )$.
The risk of the randomized algorithm $h_{\hat{\theta}_t}$ at time $t$ is then
$$
\E_{P_t} \left[
\ell_{\delta} \left (h_{\hat{\theta}_t}(X), Y \right )
\right]
= \hat{\theta}_{t,0} \delta
+ (1 - \hat{\theta}_{t,0}) \E_{P_t} \left[ \ell \left (\frac{\sum_{t' = 1}^{t}\hat{\theta}_{t,t'} G_{t'}(X)}{ \sum_{t' = 1}^{t} \hat{\theta}_{t,t'}}, Y \right ) \right]
.
$$
Over a time period of length $T$, the (realized) average risk of the pACP is
\begin{align}
\frac{1}{T}
\sum_{t = 1}^T
\E_{P_{t}}
\ell_{\delta}\left(
h_{\hat{\theta}_t}(X), Y
\right).
\label{eq:cum_risk}
\end{align}

For a given non-inferiority margin $\epsilon > 0$, our goal is to design a pACP such that its average risk will not exceed $\delta + \epsilon$ over all distributional shift patterns in $\mathcal{D}_{V,W}$ and all adaptive model development processes.
For $t = 1,\dots,T$, let $U_t$ be a vector of $n$ IID random variables drawn from the standard uniform distribution, which dictates how IID monitoring data is drawn from $P_t$.
We formally express the error-rate constraint as follows:
\begin{align}
\max_{
	\substack{
	\{p_t: t = 1,\dots, T\} \in \mathcal{D}_{V,W}\\
	\{g_t: t = 1,\dots, T\}}
}
\E_{U_1,\dots, U_T} \left[
\frac{1}{T}
\sum_{t=1}^{T}
\E_{P_{t}}
\ell_{\delta}\left(h_{\hat{\theta}_t}(X), Y\right)
\right]
\le \delta + \epsilon.
\label{eq:bound_form}
\end{align}
We would like to minimize \eqref{eq:cum_risk} as much as possible as long as \eqref{eq:bound_form} is satisfied.
We have defined this prioritization between controlling and minimizing the average risk to parallel the prioritization between Type I and II errors in hypothesis testing.

\section{Learning how to approve}
\label{sec:methods}

Since the best approval strategy is highly dependent on the problem setting, our approach is to design a pACP that learns \textit{how} to approve modifications, which we refer to as L2A.
To this end, we introduce a family of approval strategies that unifies a wide range of strategies that vary in their prior belief about the stationarity of the data and reliability of the model development process.
L2A learns an appropriate strategy by searching over this family.

\subsection{A family of approval strategies}
\label{sec:family}


In this section, we show how approval strategies can be defined as a sequence of penalized empirical risk minimization (ERM) problems.
The objective function in these ERM problems includes two regularization terms: one expresses our optimism in modifications proposed by the model developer, and the other expresses our optimism in our ability to predict future performance using historical data.
By varying the hyperparameters that scale these two regularization terms, we recover a wide range of strategies that includes or approximates previously proposed strategies.
Our formulation is inspired by the optimistic mirror descent algorithm \citep{Rakhlin2013-rj} but considers a much wider range of ``optimism.''

We denote the set of all possible sequences of hard approvals up to time $T$ using $\mathcal{S}_T$.
For convenience, suppose the hard approval sequences are ordered, so that $s_j \in \mathcal{S}_T$ denotes the $j$th hard approval sequence.
The $t$-th element of $s_j$, denoted $s_{j,t}$, can take on values in $\{0,\dots,t\}$, where zero corresponds to choosing the abstention-only model $G_0$.
To accommodate sequences of soft approvals, let $\Xi_T$ be the probability simplex over $\mathcal{S}_T$.
For $\xi \in \Xi_T$, $\xi_j$ is the probability assigned to hard approval sequence $s_j$.

At every time point, the approval strategy solves a penalized ERM problem over all soft approval sequences $\Xi_T$.
Given the solution $\hat{\xi}$ at time $t$, the approval status of modification $k$ at time $t$ is $\sum_{s_j \in \mathcal{S}_T} \hat{\xi}_{j} \mathbbm{1}\{s_{j,t} = k \}$.

The objective function of the ERM problem involves two regularization terms, $\mathcal{R}_{\eta_1}(\xi)$ and $\xi^\top M_{t}$, that represent our prior beliefs about the problem setting.
The function $\mathcal{R}_{\eta_1}: \Xi_T \mapsto \mathbb{R}$ for $\eta_1 \in [0,1]$ characterizes our prior belief that modifications from the model developer are beneficial.
A small $\eta_1$ encodes the prior belief that modifications are rarely beneficial and should be approved after thorough evaluation; a large $\eta_1$ encodes the prior belief that most modifications are beneficial and should be approved readily.
The second term $\xi^\top M_{t}$ represents our prior belief that future performance can be predicted using historical data and is scaled by a penalty parameter $\eta_2 > 0$.
The vector $M_t$ contains upper confidence bounds (UCBs) for the risk of all hard approval sequences, and the inner product $\xi^\top M_{t}$ is the bound for soft approval sequence $\xi$.
Because the bounds assume the data is stationary, a large value of $\eta_2$ corresponds to the prior belief that past performance predicts future performance.

We define a family of approval strategies by varying the penalty parameters $\eta_1$ and $\eta_2$ as well as a learning rate $\eta_3 \ge 0$.
At each time point, the approval strategy for hyperparameter $\boldsymbol{\eta} = (\eta_1,\eta_2,\eta_3)$ solves
\begin{align}
\begin{split}
\min_{\xi \in \Xi_T} &
\ \eta_3
\sum_{j = 1}^{|\mathcal{S}_T|} \left\{
\xi_j \sum_{t' = 1}^{t - 1} \E_{P_{t',n}} \left[
\ell_{\delta} \left (G_{s_{j,t'}}(X), Y \right )
\right]
\right \}
+ \eta_2 \xi^\top M_{t}
+ \mathcal{R}_{\eta_1}(\xi)
\\
\text{s.t.} &
\ \xi_j \le \mathbbm{1}\left\{M_{t',j} \le \delta + \tilde{\epsilon} \right\} \quad \forall t' = 1,\dots,t, \forall j = 1,\dots, t'
\end{split}.
\label{eq:approval_fam}
\end{align}
\noindent
The objective function is the weighted average of the empirical risk across all hard approval strategies with respect to $\xi$ plus the two aforementioned regularization terms.
The hyperparameters $\eta_1, \eta_2$, and $\eta_3$ specify the trade-off between minimizing these three terms.
We refer to $\eta_3$ as a learning rate since it governs how much the objective function and the solution to \eqref{eq:approval_fam} change in response to newly collected monitoring data.
The constraints in \eqref{eq:approval_fam} restrict approval to modifications with UCBs smaller than the abstention cost plus some margin $\tilde{\epsilon}$, which can be thought of as a non-inferiority margin for each time-point.
We discuss how to choose $\tilde{\epsilon}$ for \textit{overall} error-rate control in the next section.
Note that the feasible set is never empty because the abstention-only option always satisfies the constraints.

Approval strategies defined using \eqref{eq:approval_fam} differ from the optimistic mirror descent algorithm \citep{Rakhlin2013-rj} in two ways.
First, we scale the estimated risk bound $\xi^\top M_{t}$ by hyperparameter $\eta_2$, whereas the original formulation fixed $\eta_2$ to $1$.
By allowing for large values of $\eta_2$, our family of approval strategies spans a wider range of optimism, including those that would only be suitable for stationary settings.
Second, the added constraints disallow approval of modifications that performed poorly at previous time points.
In contrast, the original formulation takes a less aggressive approach and only down-weights such models.
However, we find that removing poor-performing modifications tends to be more effective in practice, since model performance is usually worse at later time points, not better.

\begin{figure}
	\centering
	\includegraphics[width=0.4\textwidth]{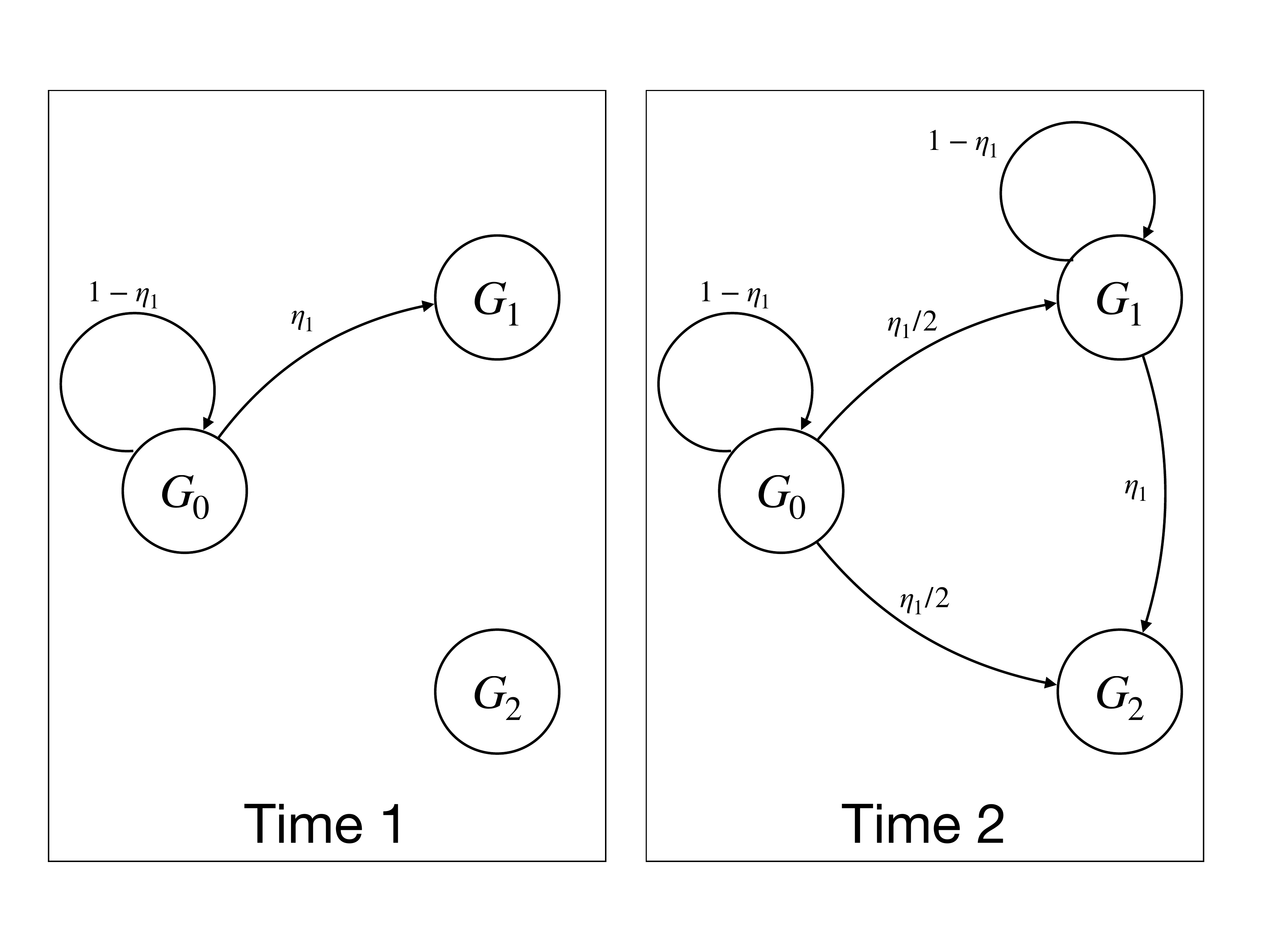}
	\caption{
		Example Markov chain prior over hard approval sequences for hyperparameter $\eta_1 \in [0,1]$.
		Node $G_0$ corresponds to the abstention-only model; $G_1$ and $G_2$ are the proposed modifications at times $t = 1$ and $2$, respectively.
		The arrows correspond to transition/approval probabilities and are labeled with transition probabilities, where an arrow returning too the current state means that nothing was approved.
		No arrows can enter $G_2$ at time $t = 1$ because it is not yet available.
		This prior approves a new modification with probability $\eta_1$, only allows transitions to the current modification or later modifications, and weights all later modifications equally.
	}
	\label{fig:markov_hedge}
\end{figure}

In general, solving \eqref{eq:approval_fam} is computationally intractable since it requires searching over all possible approval sequences.
Nevertheless, the approval status at time $t$ can be computed efficiently for certain choices of $\mathcal{R}_{\eta_1}$.
We draw inspiration from the MarkovHedge algorithm \citep{Shalizi2011-gh, Mourtada2017-dr}, which uses dynamic programming to efficiently search over all possible approval sequences.
The MarkovHedge places a Markov chain prior over all hard approval sequences with initial distribution $a \in \Delta^{T + 1}$ and transition probability matrices $A^{(2)},...,A^{(T)}$, where the accessible states at time $t$ are $G_0$ through $G_t$ (Figure~\ref{fig:markov_hedge}).
Let $\eta_1$ be the probability that a new modification is approved in this Markov prior.
We show in Section~\ref{sec:fam_proof} of the Supplement that the MarkovHedge algorithm corresponds to solving a sequence of penalized ERM problems with the regularization term
\begin{align}
\mathcal{R}_{\eta_1}(\xi) =
\left( \sum_{j=1}^{|\mathcal{S}_T|} \xi_j \log \xi_j \right ) -
\sum_{j=1}^{|\mathcal{S}_T|}
\xi_j
\left( \ln a_{s_{j,1}} + \sum_{t = 2}^T \ln A^{(t)}_{s_{j,t}, s_{j,t - 1}} \right ).
\label{eq:reg_markovhedge}
\end{align}
We use this same definition for $\mathcal{R}_{\eta_1}$ in \eqref{eq:approval_fam}.
The first summation in \eqref{eq:reg_markovhedge} is entropic regularization.
The second summation is the log probability of the soft approval sequence with respect to the Markov prior and encourages approving modifications at a rate that is highly probable per this prior.

\begin{algorithm}
	\SetAlgoLined
	\DontPrintSemicolon
	Initialize $\tilde{\theta}_{1,0} = a_0$ and $\tilde{\theta}_{1,1} = a_{1}$. \;
	\For{$t = 1,...,T$}{
		\For{$j = 0,...,t$}{
			Set $\tilde{M}_{t,j}$ to the upper confidence bound for the risk of modification $G_j$.\;
			Set 	$
			\hat{\theta}_{t,j} = \frac{
				\tilde{\theta}_{t,j} \exp(-\eta_2 \tilde{M}_{t,j} ) \mathbbm{1}\{ \tilde{M}_{t,j} \le \delta + \tilde{\epsilon} \}
			}{
				\sum_{j' = 0}^{t} \tilde{\theta}_{t,j'} \exp(-\eta_2 \tilde{M}_{t,j'} ) \mathbbm{1}\{ M_{t,j'} \le \delta + \tilde{\epsilon} \}
			}
			$.
			\tcp*{Output optimistic weights}
		}
		Deploy the stochastic selective prediction model $h_{\hat{\theta}_{t}}$ that predicts using the ensemble $\frac{1}{\sum_{j=1}^{t} \hat{\theta}_{t,j}} \sum_{j=1}^{t} \hat{\theta}_{t,j} G_{j}$ with probability $1 - \hat{\theta}_{t,0}$ and abstains with probability $\hat{\theta}_{t,0}$.\;
		Observe monitoring data with empirical distribution $P_{t,n}$.\;
		Query the model developer for a new modification $G_{t + 1}$.\;
		\For{$j = 0,...,t$}{
			Set $
			v_{t,j} = \frac{
				\tilde{\theta}_{t,j} \exp(-\eta_3 \E_{P_{t,n}} \ell_{\delta}(G_{j}(X), Y) )
			}{
				\sum_{j' = 0}^{t} \tilde{\theta}_{t,j'} \exp(- \eta_3 \E_{P_{t,n}} \ell_{\delta}(G_{j'}(X), Y) )
			}
			$.
			\tcp*{Update weights using monitoring data}
		}
		Set $\tilde{\theta}_{t + 1} = A^{(t + 1)} v_{t, j}$. \tcp*{Apply transition matrix from Markov prior}
	}
	\caption{Approval strategy defined in \eqref{eq:approval_fam} with penalty parameters $\eta_1$ and $\eta_2$, learning rate $\eta_3$, and margin $\tilde{\epsilon}$. Let $\mathcal{R}_{\eta_1}$ be the penalty for the MarkovHedge algorithm based on a Markov chain with initial distribution $a$ and transition probability matrices $A^{(2)},\cdots,A^{(T)}$.}
	\label{algo:fam}
\end{algorithm}

In Section~\ref{sec:fam_proof} of the Supplement, we prove that approval statuses from strategy $\boldsymbol{\eta}$ can be computed efficiently using Algorithm~\ref{algo:fam}.
Rather than solving \eqref{eq:approval_fam} directly, we use a variant of the MarkovHedge algorithm  to update the approval status of the proposed modifications.
In particular, Algorithm~\ref{algo:fam} looks one step ahead at each time $t$ and reweights each model $G_j$ based on its estimated risk bound $\tilde{M}_{t,j}$.
Based on Assumption~\ref{assume:drift}, we construct confidence bounds for each modification based on prospective monitoring data collected within the last $W$ time points.
Thus, $\tilde{M}_{t,j}$ bounds the risk of $G_j$ with respect to $P_{\tau_{t,j}:t-1}$, where $\tau_{t,j}$ denotes the start index of this time window.
For the special case where $j = t$, we set $\tau_{t,t} = t-1$ and construct confidence bounds using a training/validation split or cross-validation.

Next, we highlight approval strategies that are special cases within this family.

\paragraph{Abstention only}
The most pessimistic strategy refuses to use any models from the developer and only abstains.
This is an important special case: When distributional shifts are severe, none of the candidate modifications have acceptable performance, and the abstention-only strategy is the safest option.
This corresponds to \eqref{eq:approval_fam} with $\eta_1 = \eta_2 = \eta_3 = 0$.
That is, the Markov chain prior assigns a probability of one to the approval sequence that uniformly abstains over time and zero to all others.
This strategy ignores the empirical loss of the proposed modifications.

\paragraph{Blind approval}
The most optimistic validation strategy approves the latest modification instantaneously.
We approximate this strategy by setting $\eta_1$ close to one and $\eta_2 = \eta_3 = 0$.
This will approve the latest modification at each time step as long as its risk bound satisfies the constraint in \eqref{eq:approval_fam}.

\paragraph{Repeated T-tests against baseline}
If we believe that distributional shifts are rare, an effective approval strategy is to identify modifications that are non-inferior to the abstention option by hypothesis testing and to select the one with the best historical performance.
We approximate non-inferiority testing using the constraint in \eqref{eq:approval_fam} for $\tilde{\epsilon} > 0$ and select the modification with the lowest upper confidence bound by setting $\eta_1 = 0.5$, $\eta_3 = 0$, and $\eta_2$ to some large value (e.g. $> T$).

\paragraph{MarkovHedge}
When the distribution is highly non-stationary, we can use the MarkovHedge algorithm by setting $\eta_1 \in (0,1)$, $\eta_2 = 0$, and $\eta_3 > 0$.
At every time point, this strategy will deploy a selective prediction model with some non-zero abstention rate.

\subsection{Searching over a family of approval strategies}
\label{sec:l2a_ewaf}


L2A optimizes the approval strategy by searching over the family of strategies defined in Section~\ref{sec:family} using the EWAF algorithm.
Suppose we have $m$ candidate approval strategies with hyperparameters $\boldsymbol{\eta}^{(j)}$ for $j = 0,\dots,m - 1$, where $\boldsymbol{\eta}^{(0)}$ is the abstention-only strategy.
Let the approval status at time $t$ from strategy $\boldsymbol{\eta}^{(j)}$ be denoted $\hat{\theta}_{t}^{(j)}$.
L2A searches over the candidate strategies by dynamically assigning each strategy a probability weight, where $w_{t,j}$ is the probability weight for strategy $j$ at time $t$.
Initially, the probability weights for all strategies are equal.
For some meta-learning rate $\lambda > 0$, L2A updates the weight for the $j$th strategy at times $t = 1,\dots,T$ as follows:
\begin{align}
w_{t,j} =
\frac{
w_{t - 1, j} \exp \left  (-\lambda \E_{P_{t - 1,n}} \ell_{\delta} \left (h_{\hat{\theta}_{t - 1}^{(j)}}(X), Y \right ) \right )
}{
\sum_{j' = 0}^{m - 1} w_{t - 1,j'} \exp \left (-\lambda \E_{P_{t - 1,n}} \ell_{\delta}\left (h_{\hat{\theta}_{t-1}^{(j')}}(X), Y \right ) \right )
}.
\label{eq:l2a_weights}
\end{align}
L2A then deploys the selective prediction model $h_{\hat{\theta}_t^{\L2A}}$ where $\hat{\theta}_t^{\L2A} = \sum_{j = 1}^m w_{t,j} \hat{\theta}_t^{(j)}$.
The pseudocode for L2A is given in Algorithm~\ref{algo:l2v}.

\begin{algorithm}
	\SetAlgoLined
	\DontPrintSemicolon
	Initialize the fail-safe approval strategy with hyperparameter $\boldsymbol{\eta}^{(0)} = (0,0,0)$.\;
	Initialize $m - 1$  approval strategies $\boldsymbol{\eta}^{(1)}, \dots,\boldsymbol{\eta}^{(m-1)}$.\;
	Initialize probability weights $w_{1,0} = w_{1,1} = \cdots = w_{1,m-1} = \frac{1}{m}$ \;
	\For{$t = 1,...,T$}{
		Query the model developer for a new modification $G_{t}$.\;
		Query the $m - 1$ strategies for approval statuses $\hat{\theta}_t^{(j)}$ for $j = 1,\dots,m-1$.\;
		Deploy selective prediction model $h_{\hat{\theta}_t^{\L2A}}$ with $\hat{\theta}_t^{\L2A} = \sum_{j = 0}^{m-1} w_{t,j} \hat{\theta}_t^{(j)}$.\;
		Observe monitoring data with empirical distribution $P_{t,n}$.\;
		\For{$j = 0,...,m - 1$}{
			Set $
			w_{t + 1,j} = \frac{
				w_{t,j} \exp \left (-\lambda \E_{P_{t,n}} \ell_{\delta} \left (h_{\hat{\theta}_{t}^{(j)}}(X), Y \right ) \right )
			}{
				\sum_{j' = 0}^{m-1} w_{t,j'} \exp \left (-\lambda \E_{P_{t,n}} \ell_{\delta} \left (h_{\hat{\theta}_{t}^{(j')}}(X), Y \right ) \right )
			}
			$. \tcp*{Update weight for strategy $j$}
		}
	}
	\caption{Learning-to-approve (L2A) for $m$ approval strategies $\boldsymbol{\eta}^{(0)},\cdots, \boldsymbol{\eta}^{(m-1)}$ and learning rate $\lambda > 0$.}
	\label{algo:l2v}
\end{algorithm}

Next, we bound the average risk of L2A.
Our proof, given in Section~\ref{sec:proof_reg} of the Supplement, is based on existing bounds for the EWAF and improves on them in two ways.
First, because L2A has the option to abstain, we can use the abstention cost to tighten the risk bound. 
Second, the distributional shifts are bounded according to Assumption~\ref{assume:drift}, which lets us bound the worst-case risk at each time point with high probability.
\begin{theorem}
	Consider any sequence of distribution shift functions $\{p_t: t = 1,\dots,T \} \in \mathcal{D}_{V, W}$ for some $V > 0$ and window size $W$, and any model development process $\{g_t: t = 1,\dots, T\}$.
	Suppose there exist constants $\alpha_1, \alpha_2, z \ge 0$ such that the following probability bounds are satisfied by the risk bounds $\{M_{t,j}: t=1,\dots, T, j = 1,\dots, t\}$:
	\begin{align}
	& \Pr\left(
	\max_{j = 1,...,t} \E_{\Pp_{\tau_{t,j}:t - 1}} \ell_{\delta} \left(G_{j}(X), Y \right) - \tilde{M}_{t,j} \ge 0
	\right) \le \alpha_1
	\quad \forall t = 1,...,T
	\label{eq:pred_assum1}\\
	& \Pr\left(
	\max_{j = 1,...,t} \E_{\Pp_{\tau_{t,j}:t - 1}} \ell_{\delta} \left(G_{j}(X), Y \right) - \tilde{M}_{t,j} \ge z
	\right)
	\le
	\alpha_2
	\quad \forall t = 1,...,T.
	\label{eq:pred_assum2}
	\end{align}
	Then for any learning rate $\lambda > 0$ and $\tilde{\epsilon} \ge 0$, the expected average risk for L2A satisfies
	\begin{align}
	\hspace{-0.3in}
	\E_{U_1,\dots, U_T}
	\left[
	\frac{1}{T}
	\sum_{t = 1}^T \E_{P_t} \ell_{\delta} \left (h_{\hat{\theta}^{\L2A}_t}(X), Y \right)
	\right ]
	\le
	-\frac{1}{c(\lambda, \tilde{\epsilon}, z)}
	\left(
	\lambda \delta
	+ \frac{1}{T}\ln m
	+ \frac{\lambda^2}{8n}
	\right )
	\label{eq:loss_bound}
	\end{align}
	where
	\begin{align*}
	c(\lambda, \tilde{\epsilon}, z)
	& =
	\frac{
		\exp(-\lambda (\delta+\tilde{\epsilon} + V)) -1
	}{
		\delta+\tilde{\epsilon} + V
	}
	\left(1 - \alpha_1 -  \alpha_2 \right) +
	\frac{
		\exp(-\lambda (\delta +\tilde{\epsilon} + V + z) -1
	}{
		\delta +\tilde{\epsilon}+ V + z
	}
	\alpha_1 + \left(\exp(-\lambda) -1 \right ) \alpha_2.
	\end{align*}
	\label{thrm:regret}
\end{theorem}
\noindent
The risk bound is small if the risk bounds $\tilde{M}_t$ satisfy \eqref{eq:pred_assum1} and \eqref{eq:pred_assum2} with small values of $\alpha_1$ and $\alpha_2$.
We can do this by setting $\tilde{M}_{t,j}$ to the upper limit of the $(1 - \alpha_1/t)$-confidence interval of $\E_{\Pp_{\tau_{t,j}:t - 1}} \ell_{\delta} \left(G_{j}(X), Y \right)$, where we have adjusted for multiplicity using a Bonferroni correction.
To derive $\alpha_2$, we use Hoeffding's inequality.
Assuming that the confidence bound for $G_t$  at time $t$ is constructed using a validation set of size $n'$, then \eqref{eq:pred_assum2} is satisfied with $\alpha_2 = (T - 1) \exp(-8z^2n) + \exp(-8z^2n')$.
Because $z$ is only a variable that appears in the proof, we optimize its value to minimize the risk bound.

Using Theorem~\ref{thrm:regret}, we can choose the hyperparameters in L2A such that the average risk is controlled.
Our bound shares the same form as the classical EWAF bound of $\frac{1}{1 - \exp(-\lambda)}(\lambda \delta + \frac{1}{T} \ln m)$ (see e.g. Theorem~2.4 of \citet{cesa2006prediction}) and converges at the same rate with respect to $T$.
The difference is in the constants: our bound is scaled by $-1/c(\lambda, \tilde{\epsilon}, z)$ and the latter is scaled by $\frac{1}{1 - \exp(-\lambda)}$.
To understand the impact of this difference, we plot the two bounds for various abstention costs  and meta-learning rates in Figure~\ref{fig:regret_to_baseline}.
We set $m = 10$, $T = 50$, and, for the sake of comparison, $n = \infty$.
We assume the maximum drift is $V = 2\delta$.
We find that the curvatures of our risk bounds are significantly flatter and allow for much larger learning rates.
For example, for abstention cost $\delta = 0.15$, the minimum attainable risk bound from the classical result is $2\delta$ and corresponds to using a learning rate of $\lambda = 0.70$.
We can achieve this same risk bound using Theorem~\ref{thrm:regret} with a learning rate that is nearly three times as large, which translates to a much faster approval rate for beneficial modifications.

\begin{figure}
	\centering
	\includegraphics[width=0.4\linewidth]{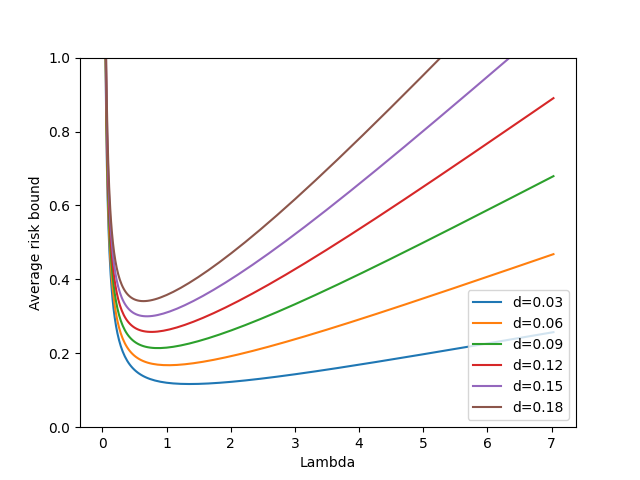}
	\includegraphics[width=0.4\linewidth]{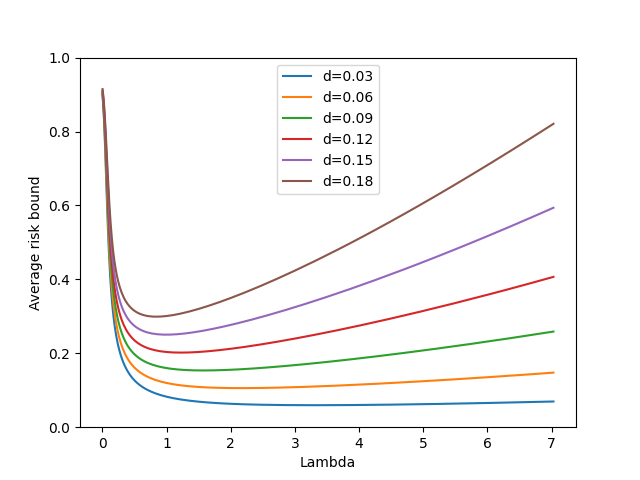}
	\caption{
		Comparison between risk bounds for $m = 10$ and $T = 50$ from Theorem~2.4 in \citet{cesa2006prediction} (left) and Theorem~\ref{thrm:regret} (right).
		The latter bound was derived under the assumption that the maximum distributional shift in terms of the maximum mean discrepancy is $V = 2\delta$.
		Each curve corresponds to a different $\delta$ value and traces out the upper bound for different $\lambda$ values.
	}
	\label{fig:regret_to_baseline}
\end{figure}


\section{Simulation study}
\label{sec:simulations}

We now perform simulation studies to evaluate L2A's ability to control and minimize the cumulative risk.
We evaluate four fixed approval strategies: abstention-only, blind approval, the MarkovHedge, and repeated T-tests.
We implemented two versions of L2A, one that only searches over the four aforementioned strategies (L2A-4) and another that searches over 12 hyperparameter settings (L2A-12).
Additional simulation details are included in Section~\ref{sec:sim_details} in the Supplement.

In the simulations below, the prediction task is binary classification, the number of time steps is $T = 50$, and the sample size at each time point is $n = 75$.
We evaluate models using the hinge loss, a convex relaxation of the zero-one loss.
At each time point, the model developer refits logistic regression on the monitoring data.
We set the abstention cost $\delta$ to the risk of the initially-approved model $G_1$.
We consider four settings with different distributions shift patterns, which all satisfy Assumption~\ref{assume:drift} for $V = \delta$ and window size $W = 3$:
\begin{enumerate}
	\item \texttt{AdaptiveShifts}:
	We simulate adaptive shifts that are targeted against the T-test approval strategy.
	In particular, the distribution is IID for a few time points and shifts whenever the T-test approves a new modification.
	In this way, the newly approved modification will perform worse than abstaining.
	The model is refit on data from the last four time points.
	\item \texttt{SmallFrequentShifts}:
	In many settings, distributional shifts are small but frequent.
	This simulation introduces small, random perturbations to the data generating mechanism every four time points.
	We simulate a model developer who tries to account for these shifts by varying the amount of data used to refit the data.
	Here, the model developer cycles between training on data from the past $j$ times points for $j = 1,\dots, 5$.
	\item \texttt{IID+GoodModels}:
	This simulation reflects the ideal problem-setting where the target population is constant over time and the proposed modifications are improving over time.
	We do this by refitting the model on all monitoring data up to the current time point.
	\item \texttt{IID+RandomModels}:
	Here the data is IID but the model development process is unreliable and most modifications are deleterious.
	To simulate this, we implement a model developer who usually refits the model only using data from the last two batches and, every four time points, refits the model on all available monitoring data.
\end{enumerate}
The full list of candidate approval strategies used in L2A-12 are listed in Table~\ref{table:l2a_candidates} of the Supplement.
We chose a wide range of candidate strategies that includes those highlighted in Section~\ref{sec:family}.
At each time point, the approval strategies constructed upper confidence bounds such that \eqref{eq:pred_assum2} was satisfied with $\alpha_1 = 0.1$.
We set the non-inferiority margin $\epsilon$ to $0.6\delta$, $\tilde{\epsilon} = 0.2\epsilon$, and used the maximum allowable learning rate according to Theorem~\ref{thrm:regret}.
$\lambda$ was set to 1.6 on average using this procedure.


\begin{figure}
	\centering
	\includegraphics[width=0.4\linewidth]{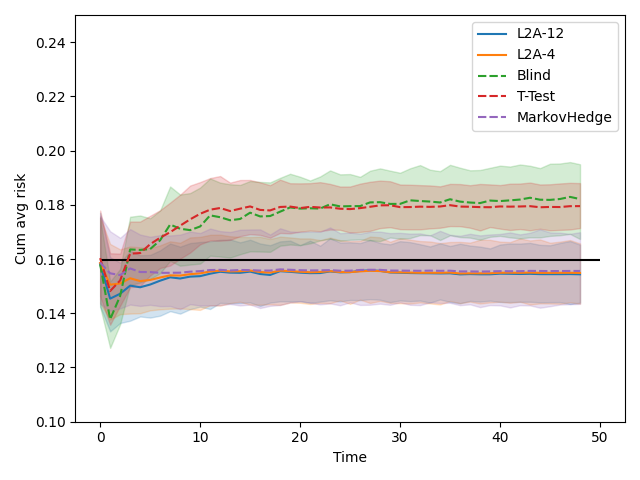}
	\includegraphics[width=0.4\linewidth]{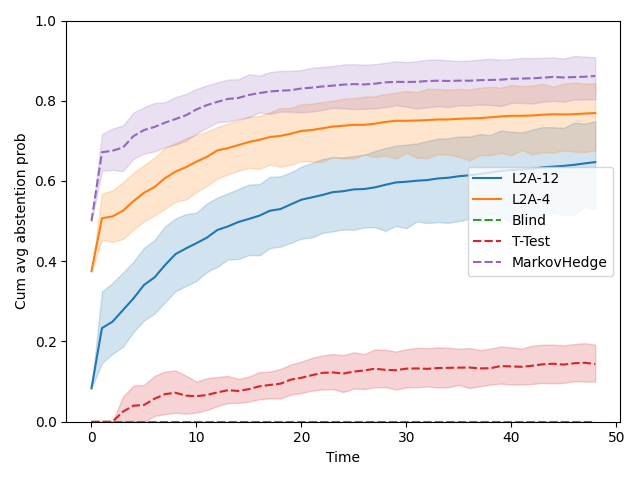}
	\\(a) \texttt{AdaptiveShifts}: Distribution shifts whenever the T-test approves a new modification.
	
	\includegraphics[width=0.4\linewidth]{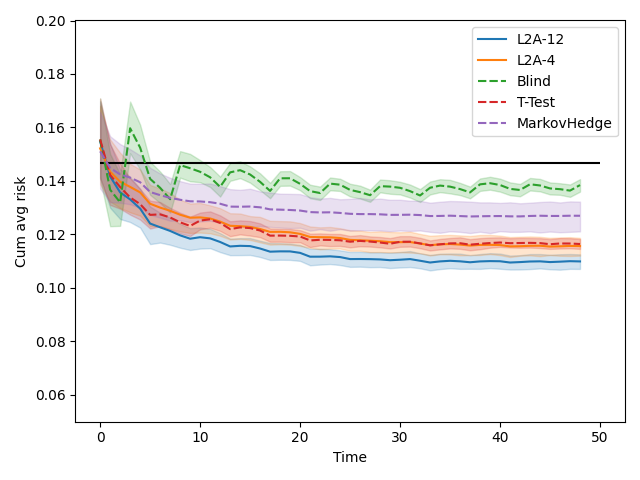}
	\includegraphics[width=0.4\linewidth]{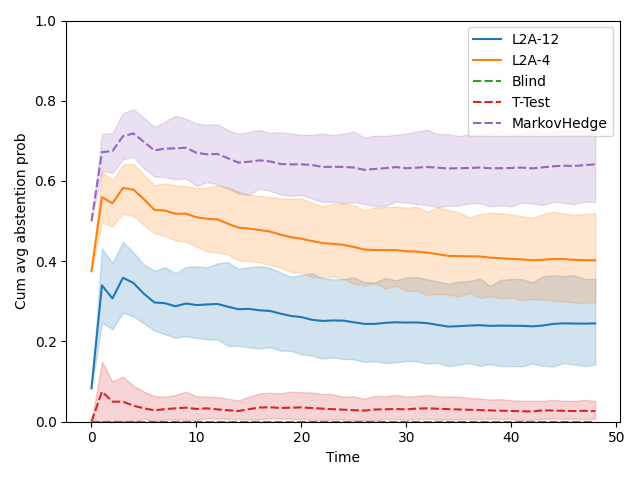}
	\\(b) \texttt{SmallFrequentShifts}: Distribution frequently shifts by a small amount.
	
	\includegraphics[width=0.4\linewidth]{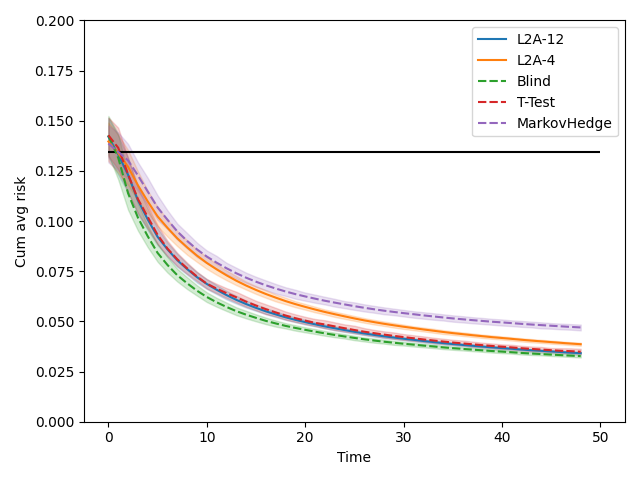}
	\includegraphics[width=0.4\linewidth]{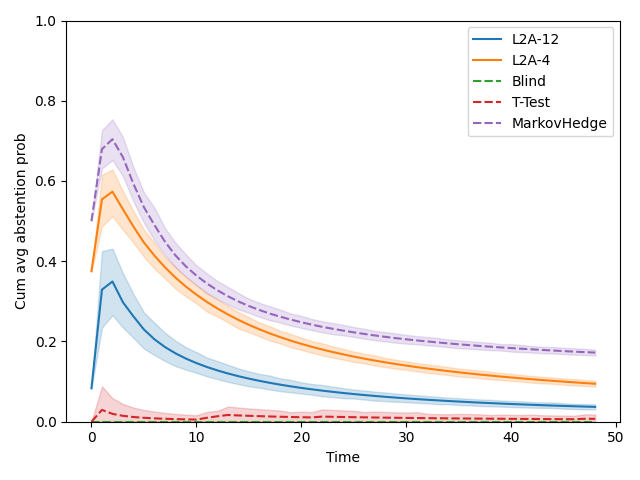}
	\\(c) \texttt{IID+GoodModels}: Data is IID. Most modifications are beneficial.
	
	\includegraphics[width=0.4\linewidth]{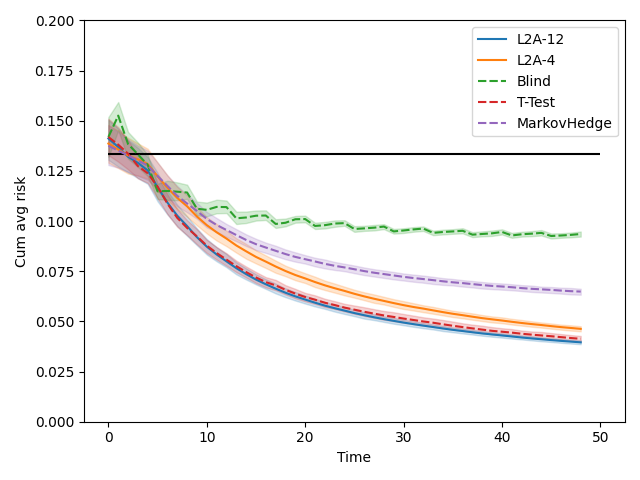}
	\includegraphics[width=0.4\linewidth]{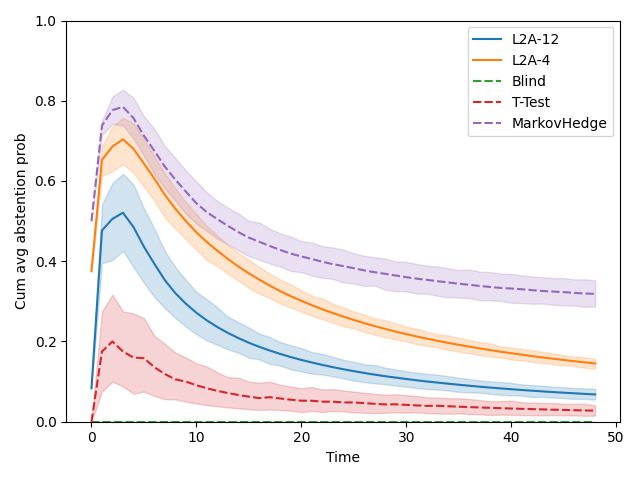}
	\\(d) \texttt{IID+RandomModels}: Data is IID. Only a few modifications are beneficial.
	
	\caption{
		Simulations comparing the cumulative average risk (left) and abstention rate (right) for different approval policies.
		The horizontal black line corresponds to the cost of abstaining.
		The solid curves correspond to the Learning-to-Approve algorithm with 4 and 12 approval strategies, which are labeled as L2A-4 and L2A-12, respectively.
		Dashed curves are fixed approval strategies: blind approval, repeated T-tests, and MarkovHedge.
		The distribution shifts over time in (a) and (b) and is IID in (c) and (d).
	}
	\label{fig:sims}
\end{figure}

Figure~\ref{fig:sims} shows the cumulative average risks and abstention rates for the pACPs.
Notably, the ranking among the four fixed approval strategies varies across the different settings.
For example, the T-test performed well in IID settings but poorly in \texttt{AdaptiveShifts}, and the MarkovHedge performed well in \texttt{AdaptiveShifts} but poorly in IID settings.
Only L2A-12 performed as good as or even better than the best fixed approval strategy across all simulation studies.

L2A-12 achieved a significantly lower risk than all the other fixed strategies in \texttt{SmallFrequentShifts}.
This is expected, because none of the four fixed strategies (or their mixtures) are optimal in this setting.
By searching over a wider range of approval strategies, L2A-12 was able to identify a better strategy.
In particular, L2A-12 steadily converged towards the hyperparameter setting of $\boldsymbol{\eta} = (0.5, 100,10)$ over time.

We also investigate how the abstention rates of the different approval strategies vary over time.
L2A increased the abstention rate in \texttt{AdaptiveShifts}, used a steady rate in \texttt{SmallFrequentShifts}, and decreased the abstention rate in the two IID simulations.
This is the desired behavior, since the abstention option is effective for defending against distributional shifts and the proposed modifications have high predictive accuracy in stationary settings.
While the MarkovHedge also learned to adjust the abstention rates over time, it learned much slower, which led to higher average risks.


\section{Experiments on real-world datasets}
\label{sec:real_data}


In this section, we evaluate how L2A performs in the presence of real-world distributional shifts.
Of course, the FDA has not yet approved continuously evolving AI/ML-based SaMDs, which means that monitoring data for these algorithms are not available.
As such, we have chosen to evaluate L2A on the Medical Information Mart for Intensive Care (MIMIC)-IV \citep{Johnson2020-iq} and Yelp (\url{http://www.yelp.com/dataset}) datasets, because (1) they span a time period of at least ten years and (2) every observation is associated with a timestamp.
For the MIMIC-IV dataset, the data has been deidentified and the timestamps are randomly shifted within a one-year time window, which will dampen distributional shifts that may have occurred in this data.
We supplement this analysis with the Yelp dataset, where the exact timestamp of each observation is available.
By analyzing both datasets, we can do a more comprehensive evaluation of L2A.
Additional details for these analyses are given in Sections~\ref{sec:yelp} and \ref{sec:mimic} of the Supplement.

\subsection{MIMIC data}
MIMIC-IV contains deidentified electronic health records for over 40,000 patients admitted to intensive care units at the Beth Israel Deaconess Medical Center between 2009 and 2019.
The dates are randomly shifted to preserve patient privacy, but the approximate year is provided to let researchers study time trends in the data.
Here we consider the benchmark task of predicting in-hospital mortality \citep{Harutyunyan2019-xs} using physiological signals, such as heart rate, pulse oximetry, and arterial blood pressure.
We simulate a model developer who trains a random forest classifier and updates it on a quarterly basis by retraining on data from the past two years.
We use the hinge loss to evaluate the models and set the abstention cost to $\delta = 0.1$.

The performance drift in the MIMIC dataset is slow, which is likely due to the robustness of our extracted features (see e.g. \citet{Nestor2019-mh}) as well as the random shifting of timestamps.
In particular, the empirical risk of the initial model increases from 0.08 to 0.1 over the ten-year time period, exceeding the abstention cost at a number of time points.
Because the performance drift in this dataset is slow, historical data can be used to retrain the ML algorithm and improve performance over time.

As shown in Figure~\ref{fig:real} (top), if we blindly approve the retrained model at every time point, the cumulative average risk drops to 0.075.
Of course, this strategy is too risky to deploy in practice.
The next best pACPs were the repeated T-tests, L2A-4, and L2A-12, which were able to keep the average risk around 0.08.
Analyzing the weights in L2A, we found that L2A steadily converged towards the blind approval strategy.
Thus, we were able to safely incorporate unsafe strategies by wrapping them within L2A.
Again, we note that the MarkovHedge algorithm performed poorly on this dataset because it relied too heavily on the abstention option.

\begin{figure}
	\centering
	\includegraphics[width=0.4\textwidth]{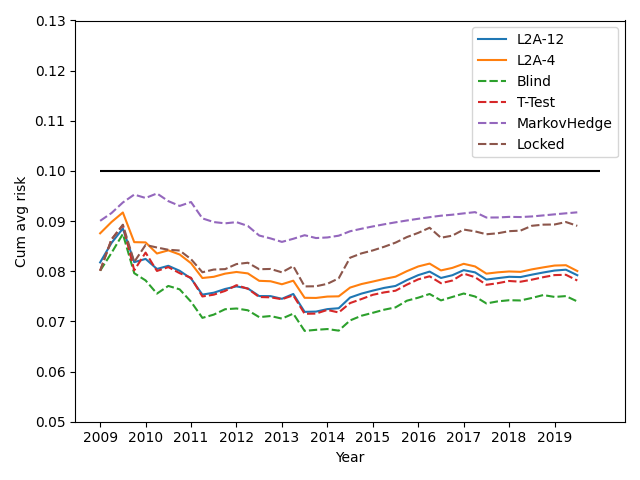}
	\includegraphics[width=0.4\textwidth]{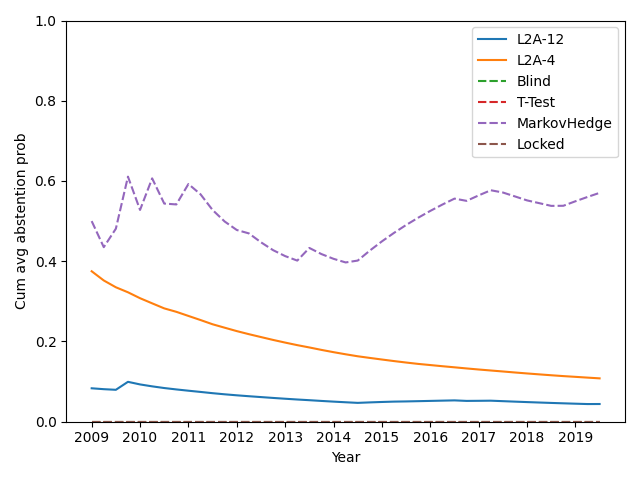}
	\includegraphics[width=0.4\textwidth]{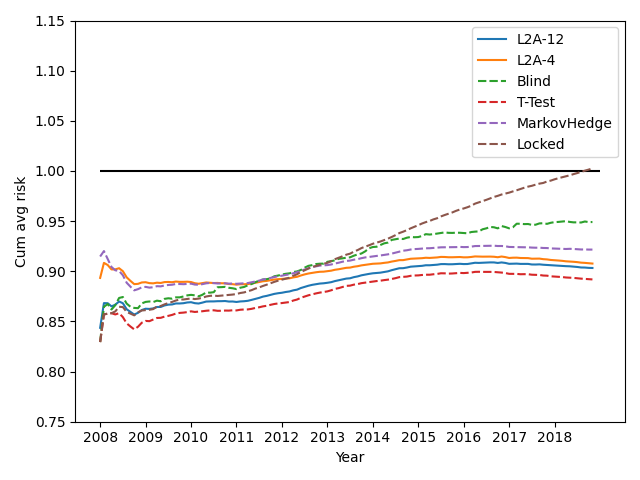}
	\includegraphics[width=0.4\textwidth]{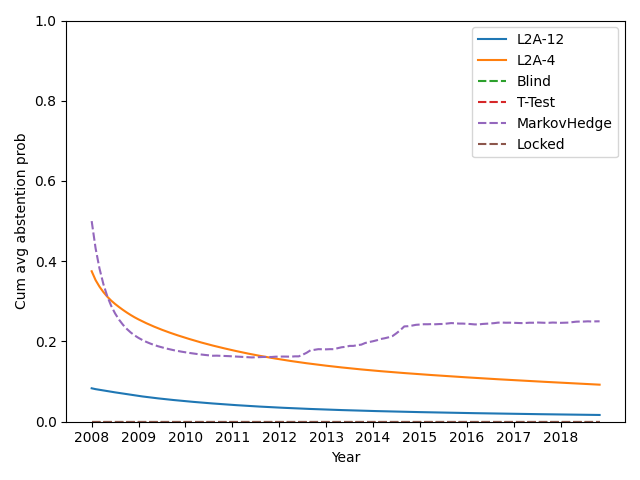}
	\label{fig:real}
	\caption{
		The cumulative average risk (left) and probability of abstaining (right) for different approval policies on the MIMIC-IV dataset (top) and the Yelp dataset (bottom).
		The ``Locked'' policy keeps the initially-approved algorithm locked.
	}
\end{figure}

\subsection{Yelp data}
The Yelp dataset contains millions of reviews left by users from 2008 to 2018, where the date and time of each review is known.
We consider the task of predicting reviewer ratings -- a number between 1 and 5 -- from the review text.
While this task is artificial, it shares similarities to many natural language processing tasks that arise in healthcare, such as electronic phenotyping \citep{Liao2015-ex, Kirby2016-zm, Zhang2019-va}.
Here, we simulate a model developer who proposes a new deep learning model every month by training on data from the prior month.
We sample 2000 reviews each month to serve as monitoring data.
We define the loss using the absolute difference and set the abstention cost $\delta$ to $1$.

There is clear evidence of performance drift in the Yelp dataset.
If we train the model only on data from January 2008 and keep it locked, its empirical risk steadily increases from 0.85 to 1.15.
As shown in Figure~\ref{fig:real} (bottom), the best approval strategy in our Yelp experiment is to repeatedly apply the T-test.
However, the T-test approval strategy is not guaranteed to perform well in the presence of distributional shifts.
By wrapping the T-test strategy within L2A, we can control the cumulative risk and recover similar operating characteristics.
Indeed, we find that L2A-12 and L2A-4 perform almost as well as the T-test strategy in this dataset.
While the fixed MarkovHedge algorithm also provides error-rate control, it is designed for adversarial setups and is therefore overly conservative in approving modifications.

\section{Discussion}

In this paper, we explored policies for evaluating and approving modifications to AI/ML-based SaMDs in the presence of adaptive but bounded distributional shifts.
Our motivation for studying adaptive distributional shifts is based on concerns that ML algorithms can affect their environments in unexpected and sometimes undesirable ways.
While shifts are unlikely to be adversarial in practice, it is important to understand if it is still possible to design approval policies with meaningful error-rate control under such weak assumptions.
Because it is difficult to anticipate which approval strategy is most appropriate for a given problem-setting, we investigated an approach that learns how to approve algorithmic updates.
We found that this learning-to-approve (L2A) approach was able to control the online error-rate without significantly sacrificing the rate at which beneficial modifications are approved.

The key reason L2A is able to control the error-rate is that it can abstain for some fixed cost $\delta$.
Without the abstention option, none of the pACPs considered in this paper are necessarily safe in the presence of distributional shifts.
In fact, keeping the initial algorithm locked is not safe either because its performance can drift over time.
In practice, one could handle abstentions by referring patients to human experts or ordering additional medical tests.

The problem of designing safe and effective pACPs for real-world settings is highly complex.
We have studied the problem of distributional shifts within a simplified framework.
Future works will need to extend our setup to more realistic settings.
For instance, this work controls a univariate performance metric, but multiple metrics need to be tracked in practice (e.g. sensitivity and specificity, or performance across different subpopulations).
In addition, we have assumed the monitoring data is representative of the target population and contains gold-standard labels.
Obtaining such ideal data is not easy to do, and this sampling problem needs to be studied in much more detail.

%

\section*{Acknowledgments}
This work was greatly improved by helpful suggestions and feedback from Alexej Gossmann, Berkman Sahiner, Brian Williamson, Noah Simon, Pang Wei Koh, and Romain Pirracchio.

\bibliographystyle{unsrtnat}
\bibliography{main}

\appendix

\begin{table}
	\centering
	\small
	\begin{tabular}{c|c}
		Notation & Description\\
		\toprule
		$T$ & The total number of time points\\
		$\mathcal{F}_t$ & Sigma algebra representing all historical data up to time $t$ \\
		$\tilde{\mathcal{F}}_t$ & Sigma algebra representing all historical data up to time $t - 1$ and proposed model at time $t$ \\
		$p_t$ & A $\mathcal{F}_{t - 1}$-measurable function that outputs the distribution at time $t$ \\
		$P_t$ & The realized output from $p_t$, the distribution at time $t$ \\
		$P_{t_1:t_2}$ & A uniform mixture of realized distributions from time $t_1$ to $t_2$, inclusive\\
		$n$ & Number of labeled observations drawn IID at each time point to serve as monitoring data\\
		$P_{t,n}$ & Empirical distribution of monitoring data at time $t$\\
		$g_t$ & A $\mathcal{F}_{t - 1}$-measurable  function that outputs a prediction model at time $t$ \\
		$G_t$ & The realized prediction model at time $t$, the output of $g_t$ \\
		$G_0$ & The selective prediction model that always abstains\\
		$\ell$ & Loss function for evaluating prediction models, assumed to be convex \\
		$\delta$ & Cost for abstaining\\
		$\ell_{\delta}$ & Augmented loss function for evaluating selection prediction models for abstention cost $\delta$\\
		$V$ & Maximum drift in the maximum mean discrepancy, used in Assumption~\ref{assume:drift}\\
		$W$ & Window size to average over, used in Assumption~\ref{assume:drift}\\
		$\theta_t$ & A $\tilde{\mathcal{F}}_{t}$-measurable function that maps to an approval status at time $t$ \\
		$\hat{\theta}_t^{\L2A}$ & The realized approval status at time $t$ from L2A\\
		$h_{\hat{\theta}_t}$ & Selective prediction model for approval status $\hat{\theta}_t$\\
		$\Xi_T$ & All possible soft approval sequences \\
		$\eta_1$ & Hyperparameter that characterizes our optimism in modifications proposed by the model developer\\
		$\mathcal{R}_{\eta_1}$ & Function that regularizas the speed at which modifications are approved\\
		$\eta_2$ & Hyperparameter that characterizes our prior belief that the distribution is stationary over time\\
		$M_t$ & Predicted risks at time $t$ for all hard approval sequences, calculated using the IID assumption\\
		$\eta_3$ & Learning rate used in an approval strategy\\
		$m$ & Number of approval strategies that L2A searches over\\
		$\boldsymbol{\eta}^{(j)}$ & Approval strategy hyperparameters considered by L2A\\
	\end{tabular}
	\vspace{0.2in}
	\caption{Summary of notation}
	\label{table:notation}
\end{table}

\section{A family of approval strategy}
\label{sec:approv_fam_deets}

The vector $M_t \in \mathbb{R}^{|\mathcal{S}_T|}$ contains the predicted risks at time $t$ for all hard approval sequences, where $M_{t,j}$ is the predicted risk of candidate modification $s_{j,t}$.
At time $t - 1$, we predict the risks as follows.
Let $\tilde{M}_{t,j}$ be the predicted risk of candidate modification $k$ at time $t$ for $k = 0,...,t$.
We set $\tilde{M}_{t,0} = \delta$ for all $t$, as the cost of abstention is known.
For $k = 0,...,t$, we set $\tilde{M}_{t,k}$ to the lower bound of the Bonferroni-corrected $(1 - \alpha_1)$-confidence interval for its expected risk for the mixture of distributions $P_{\tau_{t,k}: t - 1}$.
For $k \le t - 1$, we set $\tau_{t,k} = \max(k, t - W)$ and we construct the confidence interval using prospectively collected monitoring data.
For $k = t$, we set $\tau_{t,t} = t$ and construct the confidence interval using either a training-validation split or cross-validation.

\section{Proof for Algorithm~\ref{algo:fam}}
\label{sec:fam_proof}

In this section, we prove that Algorithm~\ref{algo:fam} outputs the approval status for \eqref{eq:approval_fam} for times $t = 1,...,T$.
We will need to introduce some notation and prove a number of lemmas before arriving at this result.

For notational convenience, let the vector $\hat{L}_{t,n}$ represent the empirical risk incurred at time $t$ for hard approval sequences in $\mathcal{S}_T$.
That is, define
\begin{align}
\hat{L}_{t,n,j} = \E_{P_{t,n}} \ell_{\delta}(G_{s_{j,t}}(X), Y) \quad \forall j = 1,...,|\mathcal{S}_T|.
\end{align}
Also, let the empirical risk incurred at time $t$ by candidate modification $G_j$ be denoted
\begin{align}
\tilde{L}_{t,n,j} = \E_{P_{t,n}} \ell_{\delta}(G_{j}(X), Y) \quad \forall j = 1,...,t.
\end{align}

Before proceeding, recall that the solution to a constrained convex optimization problem is the Bregman projection of the solution to the unconstrained optimization problem, with respect to the objective function, onto the feasible set \citep{Rakhlin2009-sx}.
More formally, it is stated as follows:
\begin{lemma}
	Suppose $\mathcal{R}: \mathbb{R}^d \mapsto \mathbb{R}$ is a strictly-convex differentiable function.
	Define convex loss functions $l_t: \mathbb{R}^d \mapsto \mathbb{R}$ for $t = 1,..,T$.
	For $w\in \mathbb{R}^d$, let $\Phi_0(w) = \mathcal{R}(w)$ and $\Phi_t(w) = \Phi_{t - 1}(w) + \eta l_t(w)$.
	Let $\mathcal{K} \subseteq \mathbb{R}^d$.
	For any $w' \in \mathbb{R}^d$, define $D_{\Phi_t}$ as the Bregman divergence with respect to $\Phi_t$ and the Bregman projection as
	\begin{align*}
	\Pi_{\Phi_t, \mathcal{K}}(w') = \arg\min_{w \in \mathcal{K}} D_{\Phi_t}(w, w').
	\end{align*}
	Then
	\begin{align*}
	\Pi_{\Phi_t, \mathcal{K}} \left (
	\min_{w \in \mathbb{R}^d} \left[\eta \sum_{s = 1}^t \ell_s(w) + \mathcal{R}(w)\right]
	\right )
	=
	\min_{w \in \mathcal{K}} \left[\eta \sum_{s = 1}^t \ell_s(w) + \mathcal{R}(w)\right].
	\end{align*}
	\label{lemma:bregman}
\end{lemma}
Using the notation in Lemma~\ref{lemma:bregman}, \eqref{eq:approval_fam} is a constrained minimization problem where $\mathcal{R}$ is $\mathcal{R}_{\eta_1}$, $\mathcal{K}$ is the probability simplex $\Xi_T$ with the $j$th entry set to zero if $M_{t,j} > \delta + \tilde{\epsilon}$, and the loss $\ell_t$ is $\E_{\Pp_{t, n}} \ell_{\delta}$.
In this case, we can calculate the Bregman projection by zeroing all entries in the unconstrained solution that are constrained to be zero and normalizing the values such that they sum to one.
For notational ease, we will use the ``is-proportional-to'' notation, $\propto$, to represent this normalization step.

We begin by showing that the original MarkovHedge algorithm (also known as the FixedShare algorithm for growing experts) \citep{Vovk1999-sm, Shalizi2011-gh, Mourtada2017-dr} can be written as the solution to a sequence of penalized empirical risk minimization (ERM) problems.
Recall that the approval status at time $t$ associated with the solution $\hat{\xi}$ is  $\sum_{s_j \in \mathcal{S}_T} \hat{\xi}_{j} \mathbbm{1}\{s_{j,t} = k \}$.
\begin{lemma}
	Define a Markov chain prior over all hard approval sequences, where the state at each time point can take on values $\{0,...,T\}$.
	Let the Markov chain have initial distribution $a \in \Delta^{T + 1}$ and transition probabilities $\{A^{(t)}: t = 2,...,T\}$, where the accessible states at time $t$ are $G_0,\dots,G_t$.
	Let $\hat{\xi}^{(t)}$ be the solution to
	\begin{align}
	\min_{\xi \in \Xi_T} &
	\ \eta
	\sum_{t' = 1}^{t - 1} \xi^\top \hat{L}_{t,n}
	+ \mathcal{R}_{\eta_1}(\xi)
	\label{eq:markov_hedge}
	\end{align}
	where $\mathcal{R}_{\eta_1}(\xi) = \sum_{j = 1}^{|\mathcal{S}_T|} \xi_j \ln \xi_j - \xi_j \left (\ln a_{s_{j,1}} +   \sum_{t = 2}^T \ln A^{(t)}_{s_{j,t}, s_{j,t - 1}} \right )$.
	The probability weights from the MarkovHedge algorithm using this Markov chain prior correspond to the approval statuses from solving \eqref{eq:approval_fam} at times $t = 1,...,T$.
	\label{lemma:markov_hedge}
\end{lemma}
\begin{proof}
	Recall that the exponentially weighted averaging forecaster (EWAF) is equivalent to the Follow The Regularized Leader (FTRL) algorithm with entropic regularization \citep{cesa2006prediction, Hazan2016-za}.
	More generally, per Lemma~\ref{lemma:bregman}, it is easy to show that the exponentially weighted averaging forecaster (EWAF) for $d$ experts with initial weights $w_{0,1}, ... w_{0,d}$, corresponds to Follow The Regularized Leader (FTRL) algorithm with the generalized entropic regularization term
	\begin{align}
	\mathcal{R}(p) = \sum_{j=1}^d p_j \ln p_j - p_j \ln w_{0,j}.
	\label{eq:reg_entrop}
	\end{align}	
	In addition, previous work has established that the outputs of the MarkovHedge for times $t = 1,...,T$ coincide with EWAF over all possible expert sequences, where the initial weights for each sequence is the probability assigned by a Markov chain prior \citep{Vovk1999-sm, Shalizi2011-gh, Mourtada2017-dr}. 
	Thus, if we define the initial weights in the regularization term \eqref{eq:reg_entrop} using the Markov chain, we arrive at our desired result.
\end{proof}


We just showed in Lemma~\ref{lemma:markov_hedge} that the MarkovHedge algorithm is the solution to a penalized ERM problem.
Next, we show how to adapt the MarkovHedge to solve the penalized ERM problem with an additional regularization term $\eta_2 \xi^\top M_t$, which also corresponds to solving \eqref{eq:approval_fam} without the additional constraints.
We introduce the following two optimization problems.
The first problem is that used in the original MarkovHedge algorithm,
\begin{align}
\min_{\xi \in \Xi_T} &
\ \eta_3 \sum_{t' = 1}^{t - 1} \xi^\top \hat{L}_{t,n}
+ \mathcal{R}_{\eta_1}(\xi),
\label{eq:approval_fam_simplest}
\end{align}
and the second introduces the term $\eta_2 \xi^\top M_t$ into the objective:
\begin{align}
\min_{\xi \in \Xi_T} &
\ \eta_3 \sum_{t' = 1}^{t - 1} \xi^\top \hat{L}_{t,n}
+ \eta_2 \xi^\top M_{t}
+ \mathcal{R}_{\eta_1}(\xi).
\label{eq:approval_fam_uncon}
\end{align}
Let the approval statuses associated with the solutions of \eqref{eq:approval_fam_simplest} and \eqref{eq:approval_fam_uncon} be denoted $\tilde{\theta}_{t}^{\simp}$ and $\hat{\theta}_{t}^{\simp,\optim}$, respectively.
We prove that their approval statuses are related as follows.

\begin{lemma}
	For all $t = 1,...,T$ and $k = 0,...,t$, we have
	\begin{align}
	\hat{\theta}_{t,k}^{\simp, \optim}
	=
	\frac{
		v_{t,k}
	}{
		\sum_{k' = 0}^t v_{t,k'}
	},
	\end{align}
	where
	\begin{align}
	v_{t,k} =
	\begin{cases}
	\exp\left(- \eta_2 M_{t,k} \right)
	\tilde{\theta}_{t,k}^{\simp}
	& t = 1
	\\
	\exp\left(- \eta_2 M_{t,k} \right)
	\sum_{k' = 0}^{t - 1}
	A^{(t)}_{k,k'}
	\exp\left(
	- \eta_3 \hat{L}_{t - 1, k'}
	\right)
	\tilde{\theta}_{t - 1,k'}^{\simp}
	& \forall t = 2,\dots ,T
	\end{cases}.
	\end{align}
	\label{lemma:optim}
\end{lemma}
\begin{proof}
	Let the solution to \eqref{eq:approval_fam_uncon} be $\hat{\xi}^{(t)}$.
	Per Lemma~\ref{lemma:bregman}, we have for $j = 1,...,|\mathcal{S}_T|$ that
	\begin{align}
	\hat{\xi}_{t,j}^{\simp, \optim}
	\propto
	\exp\left(
	- \eta_3 \left(
	\sum_{t' = 1}^{t - 1} \hat{L}_{t',n,j}
	\right)
	- \eta_2 M_{t,j}
	+ \ln a_{s_{j,1}}
	+ \sum_{t' = 2}^T \ln A^{(t')}_{s_{j,t'}, s_{j,t' - 1}}
	\right).
	\end{align}
	Aggregating over all hard sequences choosing candidate modification $k$ at time $t$, we have that the approval status for candidate modification $k$ at time $t$ is
	\begin{align}
	\hat{\theta}_{t,k}^{\simp, \optim}
	& \propto
	\sum_{\substack{
			s_j \in \mathcal{S}_T\\
			s_{j,t} = k
	}}
	\exp\left(
	-\eta_3 \sum_{t' = 1}^{t - 1} \hat{L}_{t',n, s_{j,t'}}
	- \eta_2 M_{t,k}
	+ \ln a_{s_{j,1}}
	+ \sum_{t' = 2}^T \ln A^{(t')}_{s_{j,t'}, s_{j,t' - 1}}
	\right).
	\end{align}
	So for $t > 1$,
	\begin{align}
	\hat{\theta}_{t,k}^{\simp, \optim}
	& \propto
	\sum_{\substack{
			s_j \in \mathcal{S}_t\\
			s_{j,t} = k
	}}
	\exp\left(
	- \eta_3 \sum_{t' = 1}^{t - 1} \hat{L}_{t',n, s_{j,t'}}
	- \eta_2 M_{t,k}
	+ \ln a_{s_{j,1}}
	+ \sum_{t' = 2}^t \ln A^{(t')}_{s_{j,t'}, s_{j,t' - 1}}
	\right)
	\notag
	\\
	& = 
	\sum_{k' = 0}^{t - 1}
	\exp\left(
	- \eta_3 \hat{L}_{t - 1,n, k'}
	- \eta_2 M_{t,k}
	+ \ln A^{(t)}_{k, k'}
	\right)
	\sum_{\substack{
			s_j \in \mathcal{S}_{t - 1}\\
			s_{j,t - 1} = k'
	}}
	\exp\left(
	- \eta_3 \sum_{t' = 1}^{t - 2} \hat{L}_{t', s_{j,t'}}
	+ \ln a_{s_{j,1}}
	+ \sum_{t' = 1}^{t - 1} \ln A^{(t')}_{s_{j,t'}, s_{j,t' - 1}}
	\right)
	\label{eq:optim_sol}
	\end{align}
	Moreover, from Lemma~\ref{lemma:bregman}, we know that the approval status for the solution to \eqref{eq:approval_fam_simplest} is
	\begin{align}
	\tilde{\theta}_{t - 1,k'}^{\simp}
	\propto
	\sum_{\substack{
			s_j \in \mathcal{S}_{t - 1}\\
			s_{j,t - 1} = k'
	}}
	\exp\left(
	- \eta_3 \sum_{t' = 1}^{t - 2} \hat{L}_{t', s_{j,t'}}
	+ \sum_{t' = 1}^{t - 1} \ln A^{(t')}_{s_{j,t'}, s_{j,t' - 1}}
	\right).
	\label{eq:simple_sol}
	\end{align}
	Therefore, plugging \eqref{eq:simple_sol} into \eqref{eq:optim_sol}, we have established our desired result.
	
	For $t = 1$, we have by a similar argument that
	\begin{align}
	\hat{\theta}_{t,k}^{\simp, \optim}
	& \propto
	\sum_{\substack{
			s_j \in \mathcal{S}_t\\
			s_{j,t} = k
	}}
	\exp\left(
	- \eta_2 M_{t,k}
	+ \ln a_{s_{j,1}}
	\right)\\
	& \propto
	\exp\left(
	- \eta_2 M_{t,k}
	\right )
	\tilde{\theta}^{\simp}_{t,k}.
	\end{align}
\end{proof}

Finally, we handle the constraints in \eqref{eq:approval_fam} and prove that Algorithm~\ref{algo:fam} corresponds to the approval statuses by \eqref{eq:approval_fam}.
First, let us define the following penalized ERM where we removed the regularization term $\eta_2 \xi^\top M_t$:
\begin{align}
\begin{split}
\min_{\xi \in \Xi_T} &
\ \eta_3 \sum_{t' = 1}^{t - 1} \xi^\top \hat{L}_{t,n}
+ \mathcal{R}_{\eta_1}(\xi)
\\
\text{s.t.} &
\ \xi_{t'} \preceq \mathbbm{1}\left\{M_{t'}  \preceq \delta + \tilde{\epsilon} \right\} \quad \forall t' = 1,\dots,t.
\end{split}
\label{eq:approval_fam_no_opt}
\end{align}
We denote the approval status at time $t$ from solving \eqref{eq:approval_fam_no_opt} as $\tilde{\theta}^{\constrained}_{t, k}$.

\begin{lemma}
	Algorithm~\ref{algo:fam} coincides with the approval statuses given by \eqref{eq:approval_fam} for time $t = 1,\dots,T$.
	\label{lemma:algo_fam}
\end{lemma}
\begin{proof}
	Define
	$$
	c_{t'} = \mathbbm{1}\left\{
	\xi \preceq \mathbbm{1}\left\{M_{t'} \preceq + \tilde{\epsilon} \right\}
	\right\},
	$$
	where $\preceq$ is an element-wise $\le$ and $\mathbbm{1}$ is an element-wise indicator function.
	Per Lemma~\ref{lemma:bregman}, the solution to \eqref{eq:approval_fam_no_opt} is the Bregman projection of the solution to \eqref{eq:approval_fam_uncon} over the feasible set defined by the constraints.
	As such, the approval status for \eqref{eq:approval_fam_no_opt} is defined as
	\begin{align}
	\tilde{\theta}^{\constrained}_{t, k}
	& \propto
	\sum_{\substack{
			s_j \in \mathcal{S}_t\\
			s_{j,t} = k
	}}
	\exp\left(
	- \eta_3 \sum_{t' = 1}^{t - 1} \hat{L}_{t',n, s_{j,t'}}
	+ \ln a_{s_{j,1}}
	+ \sum_{t' = 2}^{t} \ln A^{(t')}_{s_{j,t'}, s_{j,t' - 1}}
	\right)
	\prod_{t' = 1}^{t}
	c_{t',s_{j,t'}}
	\notag
	\end{align}
	For $t = 2,...,T$, we therefore have that
	\begin{align}
	\tilde{\theta}^{\constrained}_{t, k}
	& \propto c_{t,k}
	\sum_{k' = 0}^{t-1}
	\exp\left(
	- \eta_3 \hat{L}_{t - 1,n, k'}
	+ \ln A^{(t)}_{k, k'}
	\right)
	\left [
	\sum_{\substack{
			s_j \in \mathcal{S}_{t - 1}\\
			s_{j,t - 1} = k'
	}}
	\exp\left(
	- \eta_3 \sum_{t' = 1}^{t - 2} \hat{L}_{t',n, s_{j,t'}}
	+ \ln a_{s_{j,1}}
	+ \sum_{t' = 2}^{t - 1} \ln A^{(t')}_{s_{j,t'},s_{j,t' - 1}}
	\right)
	\prod_{t' = 1}^{t - 1}
	c_{t',s_{j,t'}}
	\right ].
	\end{align}
	Noticing that the quantity in the parentheses is proportional to $\tilde{\theta}_{t - 1, k'}^{\constrained}$, we have established the recursion
	\begin{align}
	\tilde{\theta}^{\constrained}_{t, k}
	& \propto
	c_{t,k}
	\sum_{k' = 0}^{t-1}
	A^{(t)}_{k, k'}
	\exp\left(
	- \eta_3 \hat{L}_{t - 1,n, k'}
	\right)
	\tilde{\theta}_{t - 1, k'}^{\constrained}
	\quad \forall t = 2,...,T.
	\notag
	\end{align}
	Using nearly the same proof as Lemma~\ref{lemma:optim}, we can show that the approval status for \eqref{eq:approval_fam} satisfies the following recursion:
	\begin{align}
	\hat{\theta}_{t, k}
	& \propto
	c_t
	\exp\left(
	- \eta_2 M_{t, k}
	\right)
	\sum_{k' = 0}^{t}
	A^{(t)}_{k, k'}
	\exp\left(
	- \eta_3 \hat{L}_{t - 1,n, k'}
	\right)
	\tilde{\theta}_{t - 1, k'}^{\constrained}
	\quad \forall t = 2,...,T.
	\notag
	\end{align}
	Using a similar argument, it is straightforward to show that
	$$
	\hat{\theta}_{1, k} =
	\exp\left(
	- \eta_2 M_{1, k}
	\right)
	\tilde{\theta}_{1, k}^{\constrained}.
	$$
	These recursions correspond to the calculations in Algorithm~\ref{algo:fam}.
\end{proof}

\section{Controlling the average risk using Learning-to-Approve (L2A)}
\label{sec:proof_reg}
For notational ease, denote the expected loss with respect to a distribution $P$ as $P \ell$ and the empirical risk with respect to an empirical distribution $P_n$ as $P_n \ell$.
Let $h_{t,j}$ denote the deployed model at time $t$ by the $j$th approval strategy.

\begin{proof}[Proof for Theorem~\ref{thrm:regret}]
	For $j = 0,...,m - 1$ and $t = 1,...,T$, define
	\begin{align}
	w_{t,j} = w_{0,j} \exp\left(
	- \lambda \sum_{t'=1}^t \Pp_{t',n} \ell_{\delta} \left(h_{t',j'}(X), Y \right)
	\right)
	\end{align}
	and $w_t = \sum_{j=0}^{m - 1} w_{t,j}$.
	To get our desired result, we derive lower and upper bounds for
	\begin{align}
	\E_{U_1,\dots, U_T}\left[
	\ln w_T
	\right].
	\label{eq:ratio_tot}
	\end{align}
	
	Since the abstention-only strategy has known cost $\delta$ at every time point, we can lower bound $\ln w_T$ as follows:
	\begin{align}
	\ln w_T
	&
	= \ln \left[
	\sum_{j=0}^{m-1} w_{0,j} \exp\left(
	- \lambda \sum_{t'=1}^t \Pp_{t',n} \ell_{\delta} \left(h_{t',j}(X), Y \right)
	\right)
	\right]
	\notag
	\\
	& \ge
	\ln \left[
	w_{0,0}
	\exp(-\lambda \delta T)
	+
	(1 - w_{0,0})
	\exp\left(
	- \lambda
	T
	\right)
	\right].
	\end{align}
	
	To derive the upper bound for \eqref{eq:ratio_tot}, we use the fact that
	\begin{align}
	\ln w_T = \sum_{t=1}^T \ln \frac{w_t}{w_{t - 1}}
	\end{align}
	and bound $\ln \frac{w_t}{w_{t - 1}}$ individually.
	In particular, we can write $\frac{w_t}{w_{t - 1}}$ as
	\begin{align}
	\ln \frac{w_t}{w_{t - 1}}
	& =
	\ln \frac{
		\sum_{j=1}^m w_{t-1,j} \exp(-\lambda \Pp_{t,n} \ell_{\delta}\left(h_{t,j}(X), Y \right) )
	}{
		\sum_{j=1}^m w_{t-1,j}
	}\\
	& =
	\ln \E_{J}
	\left[
	\exp(-\lambda \Pp_{t,n} \ell_{\delta}\left(h_{t,J}(X), Y \right) )
	\right]\\
	& = 
	\ln \E_{J}
	\left[
	\exp(-\lambda \Pp_{t} \ell_{\delta}\left(h_{t,J}(X), Y \right) )
	\exp(-\lambda (\Pp_{t,n} - \Pp_{t}) \ell_{\delta}\left(h_{t,J}(X), Y \right) )
	\right],
	\label{eq:exp_j}
	\end{align}
	where $J$ as the random variable over the support $\{0,\dots,m-1\}$ such that
	\begin{align}
	\Pr\left(
	J = j
	\right)
	= \frac{w_{t - 1,j}}{\sum_{j' = 0}^{m - 1} w_{t - 1,j'}}.
	\end{align}
	Next, we bound the conditional expectation of $\ln\frac{w_t}{w_{t - 1}}$ with respect to $U_t$ given $U_1,...,U_{t-1}$.
	Note that the distribution of $J$ depends on $U_1,...,U_{t-1}$
	Thus by Jensen's inequality, we have that:
	\begin{align}
	\E_{U_t|U_1,\dots U_{t-1}}
	\left [
	\ln \frac{w_t}{w_{t - 1}}
	\right]
	& \le
	\ln \E_{J|U_1,\dots U_{t-1}}
	\left[
	\exp(-\lambda \Pp_{t} \ell_{\delta}\left(h_{t,J}(X), Y \right) )
	\E_{U_t}
	\left [
	\exp(-\lambda (\Pp_{t,n} - \Pp_{t}) \ell_{\delta}\left(h_{t,J}(X), Y \right) )
	| J
	\right]
	\right]
	\end{align}
	Because $\ell_{\delta}$ is bounded between 0 and 1, we have by Hoeffding's inequality that
	\begin{align}
	\E_{U_t} \left[
	\exp(-\lambda (\Pp_{t,n} - \Pp_{t}) \ell_{\delta}\left(h_{t,J}(X), Y \right) )
	\mid J = j
	\right]
	\le
	\exp\left(
	\frac{\lambda^2}{8 n}
	\right).
	\end{align}
	Therefore, we can bound the expectation of $\ln w_t/w_{t - 1}$ with respect to $U_1,\dots, U_t$ by
	\begin{align}
	\E_{U_1,\dots, U_t}
	\left [
	\ln \frac{w_t}{w_{t - 1}}
	\right]
	\le 
	\E_{U_1,\dots, U_{t-1}}
	\left [
	\ln
	\E_{J|U_1,\dots, U_{t-1}}
	\left[
	\exp(-\lambda \Pp_{t} \ell_{\delta}\left(h_{t,J}(X), Y \right) )
	\right]
	\right]
	+ \frac{\lambda^2}{8 n}
	\label{eq:upper_bound_frac}
	\end{align}
	We now bound the first summand on the right hand side of \eqref{eq:upper_bound_frac}.
	Let $B_t = \max_{j : M_{t,j} \le \delta + \tilde{\epsilon} } \Pp_t \ell_{\delta}(h_{t,j}(X), Y)$.
	Consider three cases:
	\begin{align}
	&B_t \le \delta + \tilde{\epsilon} + V \label{eq:b1} \\
	\delta + \tilde{\epsilon} + V \le &B_t \le \delta + \tilde{\epsilon} + V  + z \label{eq:b2} \\
	\delta + \tilde{\epsilon} + V  + z \le&  B_t \label{eq:b3}
	\end{align}
	We use Bernstein's inequality, which states that any RV $X$ taking values in $[0,1]$ satisfies
	\begin{align}
	\ln \E[\exp(sX)] \le (\exp(s) - 1) \mathbb{E} X \quad \forall s \in \mathbb{R}.
	\end{align}
	By splitting $B_t$ into three cases, we have by Bernstein's inequality that
	\begin{align}
	\E_{U_1,...,U_{t - 1}} \left[
	\ln
	\E_{J}
	\left[
	\exp \left (-\lambda \Pp_{t} \ell_{\delta}\left(h_{t,J}(X), Y \right)  \right )
	\right]
	\right ]
	& \le
	\E_{U_1,...,U_{t - 1}} \left[
	c_t(\lambda, \tilde{\epsilon}, z)
	\E_{J|U_1,\dots, U_{t-1}}
	\left[
	\Pp_{t} \ell_{\delta}\left(h_{t,J}(X), Y \right)
	\right]
	\right ]
	\label{eq:annoying_exp}
	\end{align}
	where
	\begin{align}
	c_t(\lambda, \tilde{\epsilon}, z)
	& =
	\frac{
		\exp(-\lambda (\delta+\tilde{\epsilon} + V)) -1
	}{
		\delta+\tilde{\epsilon} + V
	}
	\Pr\left( B_t \le \delta+\tilde{\epsilon} + V \right)\\
	&\  + 
	\frac{
		\exp(-\lambda (\delta +\tilde{\epsilon} + V + z) -1
	}{
		\delta +\tilde{\epsilon}+ V + z
	}
	\Pr\left( \delta+\tilde{\epsilon} + V \le  B_t \le \delta+\tilde{\epsilon} + V + z \right)\\
	&\  + \left(\exp(-\lambda) -1 \right ) \Pr\left( B_t \ge \delta + V + z \right)
	\end{align}
	We will upper bound $c_t(\lambda, \tilde{\epsilon}, z)$ by bounding the probability that $B_t$ satisfies \eqref{eq:b2} and the probability that $B_t$ satisfies  \eqref{eq:b3}.
	First, because we have assumed that the distribution shift functions satisfy Assumption~\ref{assume:drift} for constant $V> 0$ and window size $W$, then
	\begin{align}
	B_t = \max_{j : \tilde{M}_{t,j} \le \delta +\tilde{\epsilon} } \Pp_t \ell_{\delta}(h_{t,j}(X), Y)
	& \le \max_{j : \tilde{M}_{t,j} \le \delta +\tilde{\epsilon} } \Pp_{\tau_{t,j}:t - 1} \ell_{\delta}\left(G_{j}(X), Y \right) + V\\
	& \le
	\max_{j = 1,...,t} \left( \Pp_{\tau_{t,j}:t - 1} \ell_{\delta} \left(G_{j}(X), Y \right) - M_{t,j} \right)
	+ \delta+\tilde{\epsilon} + V.
	\end{align}
	Then the probability that $B_t$ satisfies \eqref{eq:b2} is bounded by
	\begin{align}
	\Pr\left( \max_{j : M_{t,j} \le \delta +\tilde{\epsilon}} \Pp_t \ell_{\delta}(h_{t,j}(X), Y) > \delta + \tilde{\epsilon} + V \right)
	& \le 
	\Pr\left(
	\max_{j = 1,...,t} \Pp_{\tau_{t,j}:t - 1} \ell_{\delta} \left(G_{j}(X), Y \right) \ge  M_{t,j}
	\right)\\
	& \le \alpha,
	\end{align}
	where the last line follows from the assumption that $M_{t,j}$ satisfies \eqref{eq:pred_assum1} for all $t = 1,\dots,T$.
	Moreover, for any $\epsilon > 0$, the probability that $B_t$ satisfies \eqref{eq:b3} can be bounded as follows
	\begin{align*}
	& \Pr\left( \max_{j : M_{t,j} \le \delta +\tilde{\epsilon}} \Pp_t \ell_{\delta}(h_{t,j}(X), Y) > \delta+\tilde{\epsilon} + V + z \right)\\
	& \le 
	\Pr\left(
	\max_{j = 1,...,t} \Pp_{\tau_{t,j}:t - 1} \ell_{\delta} \left(G_{j}(X), Y \right) - {M}_{t,j} \ge z
	\right)\\
	& \le \alpha_2,
	\end{align*}
	where the last line follows from \eqref{eq:pred_assum2}.
	Using the fact that $s \mapsto (\exp(-\lambda s ) - 1)/s$ is non-increasing, we can bound $c_t(\lambda, \tilde{\epsilon}, z)$ by
	\begin{align}
	c(\lambda, \tilde{\epsilon}, z)
	& =
	\frac{
		\exp(-\lambda (\delta+\tilde{\epsilon} + V)) -1
	}{
		\delta+\tilde{\epsilon} + V
	}
	\left(1 - \alpha_1 -  \alpha_2 \right) + 
	\frac{
		\exp(-\lambda (\delta +\tilde{\epsilon} + V + z) -1
	}{
		\delta +\tilde{\epsilon}+ V + z
	}
	\alpha_1 + \left(\exp(-\lambda) -1 \right ) \alpha_2.
	\end{align}
	Using the above result and and the fact that $\ell_{\delta}$ is convex, we have that
	\begin{align}
	\E_{U_1,...,U_{t - 1}} \left[
	\ln
	\E_{J}
	\left[
	\exp \left (-\lambda \Pp_{t} \ell_{\delta}\left(h_{t,J}(X), Y \right)  \right )
	\right]
	\right ]
	\le c(\lambda, \tilde{\epsilon}, z) 
	\E_{U_1,...,U_{t - 1}} \left[
	\Pp_t \ell_{\delta} \left(h_{\hat{\theta}_t^{\L2A}}(X), Y\right)
	\right ].
	\end{align}
	Plugging this result in \eqref{eq:upper_bound_frac} and summing over all $t = 1,\dots,T$, we have that
	\begin{align}
	\E_U \left[\ln w_T \right ]
	\le c(\lambda, \tilde{\epsilon}, z)
	\E_{U_1,...,U_{T}} \left[\sum_{t = 1}^T \Pp_{t} \ell_{\delta} \left (h_{\hat{\theta}_t^{\L2A}}(X), Y \right ) \right ]
	+ \frac{\lambda^2T}{8n}.
	\end{align}
	Combining the lower and upper bounds, we have that
	\begin{align}
	\begin{split}
	\E_{U_1,...,U_{T}} \left[
	\sum_{t = 1}^T \Pp_{t} \ell_{\delta} \left (h_{\hat{\theta}_t^{\L2A}}(X), Y \right )
	\right ]
	&\le
	- \frac{1}{c(\lambda, \tilde{\epsilon}, z)}
	\left(
	\lambda \delta T
	- \ln w_{0,0}
	+ \frac{\lambda^2T}{8n}
	\right )
	\end{split}.
	\label{eq:emp_bound_raw}
	\end{align}
\end{proof}

\section{Simulation study details}
\label{sec:sim_details}

\begin{table}
	\begin{tabular}{cc}
		\begin{tabular}{ccc}
			$\eta_1$ & $\eta_2$ & $\eta_3$\\
			\toprule
			0 &0 &0 \\
			0 & 0 & 0.99\\
			0.5 & 10000 &0\\
			0.3 & 0 & 1.5
		\end{tabular}
		&
		\begin{tabular}{ccc}
			$\eta_1$ & $\eta_2$ & $\eta_3$\\
			\toprule
			0 &0 &0 \\
			0 & 0 & 0.99\\
			0.5 & 10000 &0\\
			0.3 & 0 & 10 \\
			0.3 & 10 & 10 \\
			0.3 & 100 & 10 \\
			0.5 & 0 & 10 \\
			0.5 & 10 & 10 \\
			0.5 & 100 & 10\\
			0.8 & 0 & 10 \\
			0.8 & 10 & 10 \\
			0.8 & 100 & 10
		\end{tabular}
	\end{tabular}
	\vspace{1in}
	\caption{Candidate approval strategy hyperparameters considered in L2A-4 (left) and L2A-12 (right).}
	\label{table:l2a_candidates}
\end{table}


Simulation results are based on running 15 replicates.
L2A-4 and L2A-12 were run with the hyperparameters shown in Table~\ref{table:l2a_candidates}.
In both cases, the first expert was the fail-safe option, which corresponds to the hyperparameter $\boldsymbol{\eta} = (0,0,0)$.

For all empirical analyses, we define the Markov chain prior in the MarkovHedge so that the approvals are monotonic in that the hard approval sequences is not allowed to transition to previous modifications.
That is, the allowed transitions at each time point are to later modifications or stay at the current modification.
More specifically, the transition probabilities at time $t$ for $k = 1,...,t-1$ are
\begin{align}
A^{(t)}_{j, k} = 
\begin{cases}
1 - \eta_1 & \text{if } j = k\\
\frac{\eta_1}{t - k} & \text{if } j =k + 1,\dots, t\\
0 & \text{otherwise}
\end{cases}.
\label{eq:transition_markovhedge}
\end{align}
The initial distribution for the Markov chain at time $t = 1$ was split evenly between the abstention option and the initial model.


\section{Yelp data analysis details}
\label{sec:yelp}
The Yelp dataset was downloaded from \url{http://www.yelp.com/dataset}.
The reviews were divided according to the month and year.
For each month, we randomly selected 2000 reviews for training the model and 2000 reviews for monitoring.
Otherwise we split half of the data for training and the other half for monitoring.

A neural network was trained to predict the number of stars assigned by the user given the review text.
Each review was represented using the mean 50-dimensional GloVE embedding \citep{Pennington2014-xy}.
The network had two hidden layers, each with 50 nodes, with ReLU activation functions.
The final output of the network was a number ranging from 1 to 5.
The network was trained using an L1-loss for 50 epochs.

\section{MIMIC data analysis details}
\label{sec:mimic}
We analyzed the ICU stay data from MIMIC-IV version 0.4.
Each patient in the MIMIV-IV dataset is associated with an \texttt{anchor\_year\_min} and \texttt{anchor\_year\_max}, which can be used to obtain the approximate year(s) the patient was treated.
We treated the average of the two years as the true year of each ICU stay and used the provided month and day information as the truth, with the understanding that these dates have been randomly shifted to some degree for deidentification purposes.
The patient stays were divided according to the quarter and year.
Three quarters of the data was randomly selected for training the models and the remaining data was used for monitoring.

We trained random forest models to predict in-hospital mortality using measured physiological signals during the first 24 hours of the patient's ICU stay.
Based on the APACHE II score \citep{Knaus1985-cx}, we extracted the following signals: heart rate, respiratory rate, blood pressure (systolic and diastolic), temperature, oxygen, blood pH, CO2 levels, sodium, creatinine, bilirubin, glucose, BUN, albumin, hematocrit, and WBC.
The mean, maximum, and minimum values were provided as features to the model.
If the signal was not measured, we imputed the values to be that from a healthy patient.
Each model trained on data from the past two years.

\label{lastpage}

\end{document}